\documentclass[preprint]{elsarticle}

\usepackage[margin=3cm]{geometry}
\usepackage{epstopdf}
\usepackage[tight,footnotesize,rm,RM]{subfigure}
\usepackage{algorithm}
\usepackage[noend]{algorithmic}
\usepackage{url}
\usepackage{framed}
\usepackage[table]{xcolor}
\usepackage{hyperref}
\usepackage{amsthm}
\usepackage{amsfonts,amssymb}
\usepackage{mathtools}
\usepackage{amssymb}
\usepackage{booktabs}
\usepackage{color}
\usepackage[font=footnotesize,skip=0pt]{caption}

\graphicspath{{./figs/}}

\newtheorem{mydef}{Definition}[section]
\newtheorem{lemm}{Lemma}[section]

\newtheorem{proposition}{Proposition}[section]
\newtheorem{property}{Property}[section]

\journal{arXiv}

\begin{document}

\begin{frontmatter}

\title{RIn-Close\_CVC2: an even more efficient enumerative algorithm for biclustering of numerical datasets}

\author[myaddress1]{Rosana~Veroneze\corref{mycorrespondingauthor}}
\cortext[mycorrespondingauthor]{Corresponding author}
\ead{veroneze@dca.fee.unicamp.br}

\author[myaddress1]{Fernando~J.~Von~Zuben}

\address[myaddress1]{University of Campinas (DCA/FEEC), 400 Albert Einstein Street, Campinas, SP, Brazil}

\begin{abstract}
RIn-Close\_CVC is an efficient (take polynomial time per bicluster), complete (find all maximal biclusters), correct (all biclusters attend the user-defined level of consistency) and non-redundant (all the obtained biclusters are maximal and the same bicluster is not enumerated more than once) enumerative algorithm for mining maximal biclusters with constant values on columns in numerical datasets. Despite RIn-Close\_CVC has all these outstanding properties, it has a high computational cost in terms of memory usage because it must keep a symbol table in memory to prevent a maximal bicluster to be found more than once. In this paper, we propose a new version of RIn-Close\_CVC, named RIn-Close\_CVC2, that does not use a symbol table to prevent redundant biclusters, and keeps all these four properties. We also prove that these algorithms actually possess these properties. Experiments are carried out with synthetic and real-world datasets to compare RIn-Close\_CVC and RIn-Close\_CVC2 in terms of memory usage and runtime. The experimental results show that RIn-Close\_CVC2 brings a large reduction in memory usage and, in average, significant runtime gain when compared to its predecessor.
\end{abstract}

\begin{keyword}
Biclustering in numerical datasets  \sep Maximal biclusters \sep Efficient enumeration
\end{keyword}

\end{frontmatter}

\section{Introduction}
\label{sec:intro}

Biclustering is a powerful data analysis technique and has been successfully applied in various application domains, such as analysis of gene expression data, text mining, collaborative filtering, treatment of missing data, and dimensionality reduction. However, due to the complexity of the biclustering problems, most of the proposed biclustering algorithms are heuristic-based, leading to sub-optimal solutions.

Even so, in the areas of Formal Concept Analysis (FCA), Frequent Pattern Mining (FPM), and graph theory, we have plenty of algorithms for enumerating all maximal biclusters with constant values (CTV) in a binary dataset. These maximal CTV biclusters are called formal concepts in FCA, closed frequent itemsets in FPM (being more specific, a closed frequent itemset corresponds to the column-set of a bicluster), and maximal bicliques in graph theory. Some examples of these algorithms are: Makino and Uno \cite{MakinoEtAl2004}, Eppstein \emph{et al.} \cite{EppsteinEtAL2010}, Close-by-One (CbO) \cite{Kuznetsov1999}, In-Close \cite{Andrews2009}, In-Close2 \cite{Andrews2011}, In-Close5 \cite{Andrews2018}, FCbO \cite{Krajca2010}, CHARM \cite{ZakiEtAL2002}, and LCM \cite{UnoEtAL2004}. Their enumeration process is characterized by being:

\begin{enumerate}
	\item Efficient: it takes polynomial time per pattern, i.e., it takes polynomial time to enumerate the first bicluster and takes polynomial time between the enumeration of two consecutive biclusters. It is the best one can computationally do in such scenario. If done properly, such algorithm will have time complexity linear in the number of biclusters and polynomial in the input size. Moreover, if the number of maximal biclusters is polynomial in the input size, the overall algorithm will be a polynomial time algorithm.
	\item Complete: it finds all maximal biclusters in a dataset. A complete enumeration guarantees to include the results produced by any other biclustering solution (given the same restrictions of internal consistency and size). So, such biclustering solution is at least of equal quality, but probably of better quality, when compared with the solution provided by any other contender. 
	\item Correct: all biclusters obey the user-defined measure of internal consistency. For instance, in the case of the aforementioned algorithms, all biclusters are submatrices of ones.
	\item Non-redundant: all biclusters are maximal and it does not enumerate the same maximal bicluster more than once. It is a very important property because the number of biclusters produced from a dataset can be very large. So, it is useful to identify the smallest representative set of biclusters from which all other biclusters can be derived \cite{TanEtAl2005}. The set of all maximal biclusters is necessary and sufficient to capture all the information about the biclusters, and has a much smaller cardinality than the set of all attainable biclusters \cite{Zaki2000}. It is important to note that the algorithm must have a smart solution to avoid redundancy, otherwise it will not be efficient. For instance, a procedure to be avoided is to check if a new bicluster is not redundant by comparing with all previously mined biclusters.
\end{enumerate}

Recently, Veroneze \textit{et al.} \cite{VeronezeEtAl2017} proposed a family of algorithms, called RIn-Close, also exhibiting these four key properties when enumerating biclusters in numerical (not only binary, but also integer or real-valued) data matrices. It may be considered a significant achievement, given that, before the RIn-Close family of algorithms, finding biclusters in numerical data matrices was accomplished by algorithms not exhibiting those four properties, or by discretizing and itemizing the numerical matrix, ultimately treating binary matrices. Notice that the RIn-Close family of algorithms is capable of mining perfect and perturbed biclusters with constant values on rows (CVR) and constant values on columns (CVC), and also perfect biclusters with coherent values (CHV). There is also an algorithm to enumerate perturbed CHV biclusters, but in this case the algorithm does not take polynomial time per bicluster.

Even though Veroneze \textit{et al.} \cite{VeronezeEtAl2017} proposed a number of improvements when we talk about the enumeration of biclusters in numerical datasets, there are still improvements that would bring many benefits to the users of RIn-Close algorithms. We focus here on the RIn-Close\_CVC algorithm. This algorithm has already been applied with great success in gene ontology enrichment analysis \cite{Veroneze2016}, analysis and identification of biomarkers \cite{Veroneze2016}, and identification of discriminative patterns in labeled datasets \cite{VeronezeVonZuben2017report}. RIn-Close\_CVC is very efficient in terms of runtime, but has a high computational cost in terms of memory usage. It must keep a symbol table in memory, whose keys are the row-sets of each found bicluster, to prevent a maximal bicluster to be found more than once. Thus, the challenge is to find some alternative to this symbol table, characterized by a lower computational cost in terms of memory, but without losing the algorithm efficiency.

We propose a new version of RIn-Close\_CVC, named RIn-Close\_CVC2, that does not use a symbol table to prevent redundant biclusters, and keeps these four key properties. The experimental results show that in addition to a large reduction in memory usage, the new algorithm also brings in average significant runtime gain, even having a higher worst-case time-complexity. Another original contribution of this work is to show formally that RIn-Close\_CVCP, RIn-Close\_CVC and RIn-Close\_CVC2 have the four aforementioned properties.

The remainder of the paper is organized as follows. Section~\ref{sec:bic} introduces definitions and mathematical notation for biclustering, and gives a short biclustering view of FCA. Section~\ref{sec:recap} reviews the algorithms In-Close2 \cite{Andrews2011}, RIn-Close\_CVCP \cite{VeronezeEtAl2017}, and RIn-Close\_CVC \cite{VeronezeEtAl2017}. Section~\ref{sec:rinclosecvc2} presents the main contribution of this paper: the new version of RIn-Close\_CVC. Experimental results are discussed in Section~\ref{sec:expresults} , and we conclude in Section~\ref{sec:conclusion}.

\section{Biclustering}
\label{sec:bic}
The formalism used here to describe a bicluster is based on \cite{VeronezeEtAl2017}.

Let $\mathbf{A}_{n \times m}$ be a data matrix with the row index set $X = \left \{ 1, 2,..., n \right \}$ and the column index set $Y = \left \{ 1, 2, ...,m \right \}$. Each row represents an object, and each column represents an attribute. Each element $a_{ij} \in \mathbf{A}$ represents the relationship between object $i$ and attribute $j$. We use $(X,Y)$ to denote the entire matrix $\mathbf{A}$. Considering that $I \subseteq X$ and $J \subseteq Y$, $\mathbf{A}_{IJ} = (I, J)$ denotes the submatrix of $\mathbf{A}$ with the row index subset $I$ and column index subset $J$.

\begin{mydef}
A bicluster is a submatrix $(I,J)$ of the data matrix $\mathbf{A}_{n \times m}$ such that the rows in the index subset $I = \left \{ i_1,..., i_k \right \}$ ($I \subseteq X$ and $k \leq n$) exhibit a consistent pattern across the columns in the index subset $J = \left \{ j_1,..., j_s \right \}$ ($J \subseteq Y$ and $s \leq m$), and vice-versa.
\label{def:bic}
\end{mydef}

Thus, a bicluster $(I,J)$ is a $k \times s$ submatrix of the matrix $\mathbf{A}$, not necessarily with contiguous rows and columns, such that it meets a certain consistency criterion \cite{MadeiraOliveira2004}. A biclustering algorithm looks for a set of biclusters $\mathfrak{B} = (I_l, J_l)$, $l = 1, ..., q$, with the total number of biclusters, $q$, being dependent on the characteristics of the selected biclustering algorithm, on the constraints imposed, and on the behaviour of the dataset being analyzed.

Considering these aspects of consistency, there are four major types of biclusters \cite{MadeiraOliveira2004}: ($i$) biclusters with constant values (CTV), ($ii$) biclusters with constant values on columns (CVC) or rows (CVR), ($iii$) biclusters with coherent values (CHV), and ($iv$) biclusters with coherent evolutions (CHE). There are many subtypes of CHE biclusters, and the order-preserving submatrix (OPSM) biclusters are the most popular among them.

Generally speaking, the bicluster enumeration problem, that we are interested in, can be described as:\\

\noindent \textbf{Problem}: To mine all maximal biclusters in a data matrix so that the obtained biclusters obey the desired level of consistency, and the enumeration process is characterized by being (1) efficient, (2) complete, (3) correct, and (4) non-redundant.\\

In this paper, we are going to focus on CTV and CVC biclusters, so we will give their definitions in what follows. See \cite{MadeiraOliveira2004, VeronezeEtAl2017} for the definitions and examples of all types of biclusters. In our definitions, a user-defined parameter $\epsilon \geq 0$ determines the maximum perturbation (residue) allowed in a consistent bicluster. Perfect biclusters are mined using $\epsilon = 0$, whereas perturbed biclusters are mined using $\epsilon > 0$. Remark that the RIn-Close family of algorithms \cite{VeronezeEtAl2017} has specialized algorithms for mining perfect biclusters, that are faster than their versions for mining perturbed biclusters.

\begin{mydef}[CTV biclusters]
A \emph{CTV bicluster} is a submatrix $(I, J)$ of a data matrix $\mathbf{A}_{n \times m}$ such that

\begin{equation}
  \max_{i \in I, j \in J} (a_{ij}) - \min_{i \in I, j \in J} (a_{ij}) \leq \epsilon \text{, with } \epsilon \ge 0.
	\label{eq:ctvbic}
\end{equation}

\label{def:ctvbic}
\end{mydef}

\begin{mydef}[CVC biclusters]
A \emph{CVC bicluster} is a submatrix $(I, J)$ such that

\begin{equation}
  \max_{i \in I} (a_{ij}) - \min_{i \in I} (a_{ij}) \leq \epsilon, \forall j \in J \text{, with } \epsilon \ge 0.
	\label{eq:cvcbic}
\end{equation}

\label{def:cvcbic}
\end{mydef}

The definition of a CVR bicluster is the equivalent transpose of the definition of a CVC bicluster. So, we can mine CVR biclusters by transposing the original data matrix and using an algorithm to mine CVC biclusters. Note that CVC and CVR biclusters are generalizations of CTV biclusters, as demonstrated in Lemma \ref{lem:bicCTVtoCVC}.

\begin{lemm}
A CTV bicluster with residue $\epsilon$ is a CVC (CVR) bicluster with residue $\epsilon'$ such that $\epsilon' \leq \epsilon$.
\label{lem:bicCTVtoCVC}
\end{lemm}

\begin{proof}
If a CTV bicluster has residue $\epsilon$, it means that the maximum pairwise variation between its elements is $\epsilon$. Given that the elements are not restricted to be part of the same column, the maximum variation in its columns may be less than $\epsilon$. Therefore, $\epsilon' \leq \epsilon$.

The proof for a CVR bicluster is equivalent.
\end{proof}

\subsection{Maximality and other properties}
\label{subsec:bicMaxProp}

\hspace{1pt}

\begin{mydef}[Maximal bicluster]
Given the desired characteristics of internal consistency (such as the user-defined maximum perturbation $\epsilon$), a bicluster $(I,J)$ is called a \emph{maximal bicluster} if and only if:
\begin{itemize}
	\item $\forall x \in X \setminus I$, $(I \cup \{x\}, J)$ is not a \textit{correct} bicluster (the bicluster does not meet the desired characteristics of internal consistency), and
	\item $\forall y \in Y \setminus J$, $(I, J \cup \{y\})$ is not a \textit{correct} bicluster.
\end{itemize}
\label{def:maximal}
\end{mydef}

\begin{property}[Anti-Monotonicity]
Let $(I,J)$ be a correct bicluster. Any submatrix $(I', J')$, where $I' \subseteq I$ and $J' \subseteq J$, is also a \textit{correct} bicluster.
\end{property}

\begin{property}[Monotonicity]
Let $(I,J)$ be a maximal correct bicluster. Any supermatrix $(I',J')$, where $I' \times J' \supset I \times J$, is not a \textit{correct} bicluster.
\end{property}

Usually, the efficient enumerative algorithms of FCA and FPM areas are based on the monotonicity and anti-monotonicity properties \citep{Besson2007}. In fact, we do not know any efficient enumerative biclustering algorithm that is not based on these properties. The given definitions of CTV and CVC biclusters are in accordance with these two properties, as well as the RIn-Close family of algorithms.

\subsection{A short biclustering view about Formal Concept Analysis}
\label{subsec:fca}

Let $\mathbf{A}_{n \times m}$ be a binary data matrix. For a subset $I \subseteq X$, we define
\begin{equation}
  I^{\uparrow} = \{y \in Y| a_{iy} = 1, \forall i \in I\}
\end{equation}
\noindent as the set of attributes common to all the objects in $I$. Similarly, for a subset $J \subseteq Y$, we define:
\begin{equation}
  J^{\downarrow} = \{x \in X| a_{xj} = 1, \forall j \in J\}
	\label{eq:blinha}
\end{equation}
\noindent as the set of objects common to all the attributes in $J$.

\begin{mydef}[Formal Concept]
A formal concept of a binary data matrix $\mathbf{A}_{n \times m}$ is a pair $(I,J)$ with $I \subseteq X$, $J \subseteq Y$, $I^{\uparrow} = J$, and $J^{\downarrow}=I$.
\end{mydef}

By the definition, we see that though many subsets $I$ can generate the same subset $J$, only the largest (closed) subset $I$ is part of a formal concept. The same reasoning is valid for $J$. Therefore, a formal concept is a maximal biclusters of 1's in a binary data matrix.

\section{Recap}
\label{sec:recap}

RIn-Close algorithms are generalizations of In-Close2 \cite{Andrews2011} to deal with numerical datasets, aiming at preserving the nice properties of In-Close2 restricted to binary datasets. More specifically, RIn-Close\_CVC \cite{VeronezeEtAl2017}, which mines perturbed CVC biclusters, is a generalization of RIn-Close\_CVCP \cite{VeronezeEtAl2017}, which mines perfect CVC biclusters. Therefore, in this section, we will review In-Close2 \cite{Andrews2011}, RIn-Close\_CVCP \cite{VeronezeEtAl2017} and RIn-Close\_CVC \cite{VeronezeEtAl2017} algorithms. Respectively, the problems that they are able to solve are:\\

\noindent \textbf{Problem 1}: To mine all maximal CTV biclusters of 1's from a binary data matrix $\mathbf{A}_{n \times m}$, with the enumeration process being (1) efficient, (2) complete, (3) correct, and (4) non-redundant.\\

\noindent \textbf{Problem 2}: To mine all maximal perfect CVC biclusters from a numerical data matrix $\mathbf{A}_{n \times m}$, with the enumeration process being (1) efficient, (2) complete, (3) correct, and (4) non-redundant.\\

\noindent \textbf{Problem 3}: Given a user-defined parameter $\epsilon \ge 0$, to mine all maximal CVC biclusters with maximum perturbation $\epsilon$ from a numerical data matrix $\mathbf{A}_{n \times m}$, with the enumeration process being (1) efficient, (2) complete, (3) correct, and (4) non-redundant.

\begin{figure*}
  \centering
	
	\subfigure[Generation of a single descendant by In-Close2.]{
		\includegraphics[trim=3.5cm 13cm 18cm 2.5cm, clip, scale=0.6]{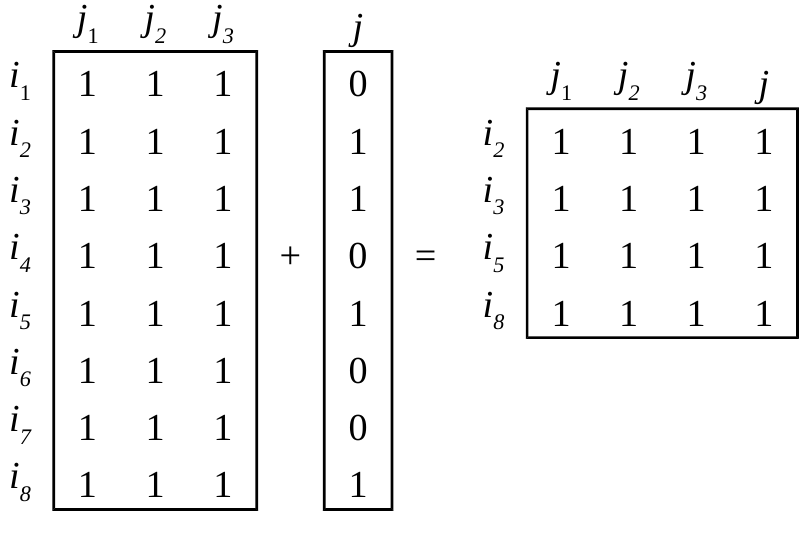}
		\label{fig:variosFilhos0}
	}
	
	\subfigure[Generation of multiple descendants, d1, d2, d3 and d4, by RIn-Close\_CVCP.]{
		\includegraphics[trim=3.5cm 13cm 17.5cm 2.5cm, clip, scale=0.6]{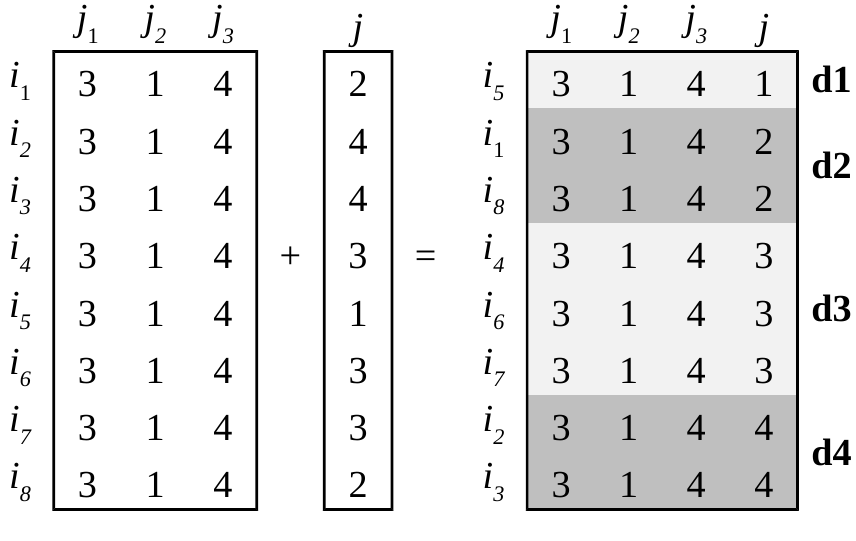}
		\label{fig:variosFilhos1}
	}
	
	\subfigure[Generation of multiple descendants, d1, d2 and d3, by RIn-Close\_CVC (considering $\epsilon = 1$).]{
		\includegraphics[trim=9.5cm 11.5cm 10.5cm 1cm, clip, scale=0.6]{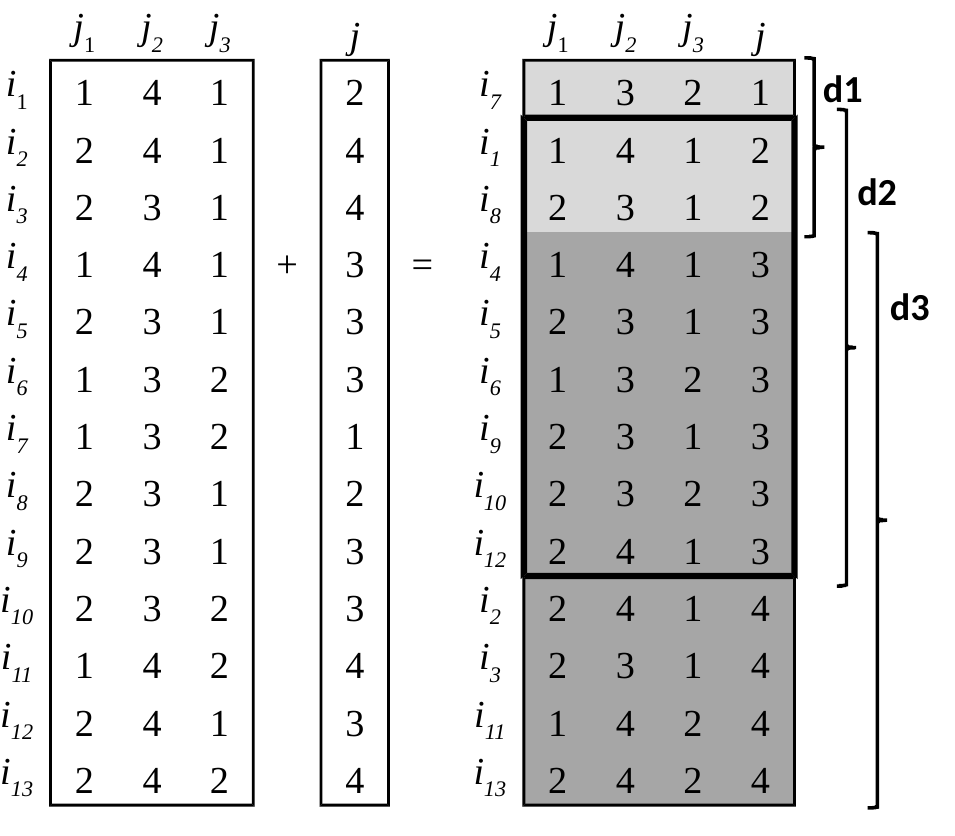}
		\label{fig:variosFilhos2}
	}
	
	\caption{Examples of the generation of descendants by (a) In-Close2, (b) RIn-Close\_CVCP, and (c) RIn-Close\_CVC (based on \cite{VeronezeEtAl2017}).}
	\label{fig:variosFilhos}
\end{figure*}

\subsection{In-Close2}
\label{subsec:inclose2}

In discrete mathematics, combinations have a lexicographical order, for instance, \{1, 2, 3\} comes before \{1, 2, 4\}, and also before \{1, 3\} \cite{Andrews2009}. Ganter \cite{Ganter1984} showed how the lexicographical order of formal concepts can be used to avoid the search of repeated results.

In-Close \cite{Andrews2009} and its successors \cite{Andrews2011,Andrews2015,Andrews2017,Andrews2018} are based conceptually on CbO \cite{Kuznetsov1999}. These algorithms use the lexicographic approach for mining formal concepts (biclusters), thus avoiding the discovery of the same formal concept more than once. They maintain a \emph{current attribute}: The next generated formal concept is new (\emph{canonical} \cite{Kuznetsov1996}) if its column-set contains no attribute preceding the current attribute.

In In-Close algorithms \cite{Andrews2009,Andrews2011,Andrews2015,Andrews2017,Andrews2018}, each formal concept $(I,J)$ is incrementally closed, i.e., its column-set $J$ is incrementally completed with all possible columns for the row-set $I$, hence the name of the algorithms: \textit{In}(cremental)-\textit{Close}(ure) \cite{Andrews2015}.

Algorithm \ref{alg:inclose2} shows In-Close2 pseudocode. It is invoked with an initial pair $(I,J) = (X,\emptyset)$ (which is called the \textit{supremum} formal concept), and an initial attribute $y = 1$. The user-defined parameter minimum number of rows, $minRow$, and the binary data matrix $\mathbf{A}_{n \times m}$ were omitted from the pseudocode input for simplicity and clarity, but they are quite obvious given the context.

During the closure of a formal concept, In-Close2 iterates across the attributes (line 1). If the current attribute $j$ is not an inherited attribute (line 2), In-Close2 computes the row-set of a candidate new formal concept $G$ (line 3). If the size of $G$ is equal to the size of $I$ (line 4), this means that the row-set $G$ is equal to the current row-set $I$, then the attribute $j$ is added to the current column-set $J$ (line 5). Otherwise, In-Close2 tests if the size of $G$ is greater than or equal to the value of the user-defined parameter $minRow$ and tests if $G$ is canonical (line 6). If yes, the current formal concept $(I, J)$ will give rise to a child formal concept, which is placed in a queue (line 7). After the closure of the current formal concept $(I, J)$, In-Close2 starts closing its descendants (lines 9 to 11). Notice that the descendants of the bicluster $(I, J)$ inherit its column-set $J$ (line 10).

\begin{algorithm}
\caption{\textit{In-Close2}\\ComputeBiclustersFrom$((I,J),y)$}
\label{alg:inclose2}
\begin{algorithmic}[1]
  \small
  \REQUIRE Formal Concept $(I,J)$ to be closed, current attribute $y$
  \FOR{$j \leftarrow y$ to $m$}
    \IF{$j \notin J$}
	  \STATE $G \leftarrow I \cap \{j\}^\downarrow$
      \IF{$|G| = |I|$}
        \STATE $J \leftarrow J \cup \{j\}$
      \ELSIF{$|G| \geq minRow$ \AND $G$ is canonical}
          \STATE PutInQueue($G, j$)
      \ENDIF
    \ENDIF
  \ENDFOR
  \STATE ProcessBicluster($I,J$)
  \WHILE{GetFromQueue($G, j$)}
    \STATE $H \leftarrow J \cup \{j\}$
    \STATE ComputeBiclustersFrom($(G,H),j+1$)
  \ENDWHILE
\end{algorithmic}
\end{algorithm}

The canonicity test is as follows. Letting $J$ be the current column-set, and $j$ be the current attribute, the row-set $G$ of a candidate new formal concept is not canonical if

\begin{equation}
    \exists k \in Y \setminus J \: | \: [k < j] \: \wedge \: [a_{gk} = 1, \forall g \in G].
	\label{eq:incloe2_iscan}
\end{equation}

In addition to the minimum number of rows $minRow$, we can easily add a minimum number of columns $minCol$ to In-Close2. While In-Close2 loops through the attributes, a formal concept $(I, J)$ can be discarded if, even adding all remaining attributes to its column-set, it will not meet the minimum number of columns $minCol$ (therefore, its next descendants will not meet the minimum number of columns $minCol$ as well). Although this restriction can be checked only during the closure of a formal concept, it will also prune the search space and save computational resources because ($i$) it stops the construction of a formal concept that will be discarded later, given that it does not meet the restriction $minCol$, and ($ii$) it avoids generating descendants that will not meet the restriction $minCol$ as well. The same is true for the RIn-Close algorithms \cite{VeronezeEtAl2017}, which are conceptually based on In-Close2.

See Andrews \cite{Andrews2015} for details and proofs associated with In-Close2, and its predecessors In-Close and CbO too.

\textbf{Worst-case complexity} - The worst-case time of checking the canonicity is $O(nm)$. Thus, the overall worst-case of In-Close2 is $O(qnm^2)$, where $q$ is the number of formal concepts in the data matrix \cite{Andrews2015}.

\textbf{Generalization of In-Close2} - Veroneze \textit{et al.} \cite{VeronezeEtAl2017} chose to generalize and extend the algorithm In-Close2 to numerical datasets because: (\emph{i}) it is easy to understand; (\emph{ii}) it is one of the fastest algorithms for enumerating maximal biclusters in binary datasets; (\emph{iii}) it has support to the desired minimum number of rows in a bicluster (the parameter $minRow$); (\emph{iv}) it is easy to incorporate a support to the desired minimum number of columns in a bicluster (a parameter $minCol$); and (\emph{v}) In-Close2 closes the biclusters incrementally, starting from the \textit{supremum} with all objects in its row-set and no attributes in its column-set. This latter aspect of In-Close2 is very important when working with integer or real-valued data matrices, as we will see in what follows. To exemplify, when mining CVC biclusters, given the current attribute, we can look for the subsets of rows of the current row-set that accomplish the maximum perturbation $\epsilon$.

\subsection{RIn-Close\_CVCP}
\label{subsec:rinclosecvcp}

The generalization of In-Close2 to enumerate all maximal perfect CVC biclusters, called RIn-Close\_CVCP, is straightforward. We have only one major difference: in In-Close2, each bicluster $(I, J)$ can generate just one descendant per attribute, whereas in RIn-Close\_CVCP, each bicluster $(I, J)$ can generate multiple descendants per attribute. It happens because In-Close2 looks for blocks of 1's, whereas RIn-Close\_CVCP looks for maximal blocks of constant values on columns.

Figure~\ref{fig:variosFilhos} illustrates this difference. In our example of Figure~\ref{fig:variosFilhos0}, In-Close2 is closing the bicluster

\begin{equation*}
(\{i_1, i_2, i_3, i_4, i_5, i_6, i_7, i_8\}, \{ j_1, j_2, j_3\}),
\end{equation*}

\noindent and it gives rises to \textbf{one} descendant in attribute $j$: $(\{i_2, i_3, i_5, i_8\}, \{ j_1, j_2, j_3, j\})$. In Figure~\ref{fig:variosFilhos1}, RIn-Close\_CVCP gives rise to \textbf{four} new perfect CVC biclusters in attribute $j$ (\emph{without overlap between their row-sets}, since RIn-Close\_CVCP looks for perfect CVC biclusters): 

\begin{itemize}
	\item (d1) $(\{i_5\}, \{ j_1, j_2, j_3, j\})$,
	\item (d2) $(\{i_1, i_8\}, \{ j_1, j_2, j_3, j\})$,
	\item (d3) $(\{i_4, i_6, i_7\}, \{ j_1, j_2, j_3, j\})$, and
	\item (d4) $(\{i_2, i_3\}, \{ j_1, j_2, j_3, j\})$.
\end{itemize}

Algorithm~\ref{alg:rinclosecvcp} shows the pseudocode of RIn-Close\_CVCP. Notice that RIn-Close\_CVCP is almost the same as In-Close2. It is also invoked with an initial pair $(I, J) = (X, \emptyset)$, and an initial attribute $y = 1$. Again, the user-defined minimum number of rows, $minRow$, and the numerical data matrix $\mathbf{A}_{n \times m}$ were omitted from the pseudocode for simplicity and clarity.

There are basically two differences between In-Close2 and RIn-Close\_CVCP pseudocodes. The first one is that the current attribute $j$ is added to the current column-set $J$ if all values of attribute $j$ and objects $I$ are equal, but not necessarily equal to 1, it may be any value. And the second one is associated with the fact that the bicluster $(I, J)$ can give rise to multiple descendants per attribute. So, RIn-Close\_CVCP computes all row-sets of the possible new biclusters and loops across them.

Letting $I$ be the current bicluster's row-set, and $j$ be the current attribute, the row-sets of the possible new biclusters are given by
 
 \begin{equation}
 \{G \: | \: [G \subseteq I] \; \wedge \; [\max_{i \in G}(\{a_{ij}\}) - \min_{i \in G}(\{a_{ij}\} = 0] \; \wedge \; [G \; \mathrm{is \; maximal}]\}.
 \label{eq:rinc_cvcp_compExt}
 \end{equation}

\noindent Note in the example of Figure~\ref{fig:variosFilhos1} that the elements of the current attribute, $j$, were sorted in order to easily identify all the row-sets of the candidate new biclusters. So, to easily compute the row-sets of the candidate new biclusters, we \textbf{sort} the values of the data matrix $\mathbf{A}$ in rows $I$ and column $j$. After sorting, we just scroll through the vector looking for the blocks of equal values.

\begin{algorithm}
  \caption{\textit{RIn-Close\_CVCP}\\ComputeBiclustersFrom$((I,J),y)$}
  \label{alg:rinclosecvcp}
  \begin{algorithmic}[1]
    \small
    \REQUIRE Bicluster $(I,J)$ to be closed, current attribute $y$
    \FOR{$j \leftarrow y$ to $m$}
      \IF{$j \notin J$}
        \IF{$\max_{i \in I}(a_{ij}) - \min_{i \in I}(a_{ij}) = 0$}
          \STATE $J \leftarrow J \cup \{j\}$
        \ELSE
          \STATE Compute the possible new row-sets \COMMENT{Eq.~\ref{eq:rinc_cvcp_compExt}}
          \FOR{each possible new row-set $G$}
            \IF{$|G| \geq minRow$ \AND $G$ is canonical}
              \STATE PutInQueue($G, j$)
            \ENDIF
          \ENDFOR
        \ENDIF
      \ENDIF
    \ENDFOR
    \STATE ProcessBicluster($I,J$)
    \WHILE{GetFromQueue($G, j$)}
      \STATE $H \leftarrow J \cup \{j\}$
        \STATE ComputeBiclustersFrom($(G,H),j+1$)
    \ENDWHILE
\end{algorithmic}
\end{algorithm}

The test of canonicity is also essentially the same as in In-Close2. Letting $J$ be the current column-set, $j$ be the current attribute, and $G$ be the row-set of a candidate new bicluster, it is not canonical if

\begin{equation}
	\exists k \in Y \setminus J \; | \; [k < j] \: \wedge \: [\max_{g \in G}(a_{gk}) - \min_{g \in G}(a_{gk}) = 0],
	\label{eq:rinc_cvc_p_iscan}
\end{equation}

\noindent i.e., if there is an attribute $k < j$ that we can add to the bicluster $(G, J)$ and it remains a correct perfect CVC bicluster.

\textbf{Worst-case complexity} - To compute the worst-case time of RIn-Close\_CVCP, we need to remember that: as we are looking for perfect CVC biclusters, there is no overlap between the row-sets of the possible new biclusters (see the example of Figure~\ref{fig:variosFilhos1}). So, if there is 1 possible new row-set, it has in the worst case at most $n$ rows; if there is 2 possible new row-sets, they have in the worst case at most $n/2$ rows each; if there is 3 possible new row-sets, they have in the worst case at most $n/3$ rows each; and so on. So, the worst-case time of lines 7 and 8 together of RIn-Close\_CVCP is $O(nm)$. Thus, the overall worst-case time of the RIn-Close\_CVCP algorithm is almost the same as the one for In-Close2: $O(qnm(\log n + m))$, where the difference is due to the use of a sorting algorithm to compute the row-sets of the candidate new biclusters (line 6). Note that usually $\log n \ll m$.

\begin{proposition}
RIn-Close\_CVCP is an (1) efficient, (2) complete, (3) correct, and (4) non-redundant algorithm for mining all maximal perfect CVC biclusters from a numerical data matrix $\mathbf{A}_{n \times m}$.
\end{proposition}
\begin{proof}
RIn-Close\_CVCP is a generalization of In-Close2, so we need only to show that the generalization steps keep these 4 properties:

\noindent (1) Efficiency: As In-Close2, RIn-Close\_CVCP has time complexity linear in the number of biclusters and polynomial in the input size.

\noindent (2) Completeness: RIn-Close\_CVCP has the same search engine as In-Close2, being the biclusters incrementally closed. The main difference here is that, in In-Close2, each bicluster $(I, J)$ can generate just one descendant per attribute, whereas in RIn-Close\_CVCP, each bicluster $(I, J)$ can generate more than one descendant per attribute. But this difference does not affect the search engine. Once the blocks of constant values of a column are identified (see Eq.~\ref{eq:rinc_cvcp_compExt}), the RIn-Close\_CVCP's procedure for each candidate new bicluster is the same as the In-Close2's procedure.

\noindent (3) Correctness: From the search engine of In-Close2 (inherited by RIn-Close\_CVCP), we can observe that a descendant bicluster of a bicluster $(I,J)$ has row-set $G \subset I$, and column-set $H \supset J$ (notice that the descendants inherit the column-set $J$ of its parent). For all $h \in H$, we must check if

\begin{equation*}
\max_{i \in G}(\{a_{ih}\}) - \min_{i \in G}(\{a_{ih}\}) = 0.
\end{equation*}

\noindent Note that we can divide the column-set $H$ of the new bicluster in three disjunct subsets: $\{j\}$, $J$, and $H \setminus \{J \cup \{j\}\}$, where $j$ is the attribute where the bicluster was created, $J$ are the set of inherited columns, and $H \setminus \{J \cup \{j\}\}$ are the set of columns that the bicluster gained during its closure. For the subset $\{j\}$ and the subset $H \setminus \{J \cup \{j\}\}$, this check is ok by construction (see Eq.~\ref{eq:rinc_cvcp_compExt} and line 3 of Algorithm~\ref{alg:rinclosecvcp}, respectively). Supposing that the parent bicluster $(I,J)$ is correct, so the bicluster $(G,J)$ is also correct because $G \subset I$. We can affirm that the parent bicluster $(I,J)$ is correct by induction, since ($i$) a bicluster having the column in which it was generated and the columns it gained during its closure is correct by construction, and ($ii$) the supremum bicluster (ancestor of all biclusters) starts with no columns to pass to its descendants.

\noindent (4) Non-Redundancy: We need to prove that the biclusters are maximal and the same (maximal) bicluster is not enumerated more than once.

As in In-Close2, the row-set of a bicluster is maximal by construction. Also, as in In-Close2, the column-set of a bicluster is maximal after its closure.

As in In-Close2, the canonicity test is responsible for the prevention of mining the same (maximal) bicluster more than once. In both algorithms, given a candidate new bicluster, it is checked if there is a column preceding the column $j$ (column in which the candidate new bicluster was generated) that could be added to the bicluster and it would remain correct. If so, based on the lexicographical order, the candidate new bicluster is ignored.
\end{proof}

\subsection{RIn-Close\_CVC with Symbol Table}
\label{subsec:rinclosecvc}

RIn-Close\_CVC is more elaborate than RIn-Close\_CVCP:  not only a bicluster $(I, J)$ is able to generate multiple descendants per attribute, but also there may be overlaps between their row-sets. For instance, in Figure~\ref{fig:variosFilhos2}, RIn-Close\_CVC is closing the bicluster

\begin{equation*}
	(\{i_1, i_2, i_3, i_4, i_5, i_6, i_7, i_8, i_9, i_{10}, i_{11}, i_{12}, i_{13}\}, \{ j_1, j_2, j_3\}),
\end{equation*}

\noindent and it gives rise to three possible new biclusters in attribute $j$ (we are considering $\epsilon = 1$):

\begin{enumerate}
\item $(\{i_7, i_1, i_8\}, \{ j_1, j_2, j_3, j\})$,
\item $(\{i_1, i_8, i_4, i_5, i_6, i_9, i_{10}, i_{12}\}, \{ j_1, j_2, j_3, j\})$, and
\item $(\{i_4, i_5, i_6, i_9, i_{10}, i_{12}, i_2, i_3, i_{11}, i_{13}\}, \{ j_1, j_2, j_3, j\})$.
\end{enumerate}

\noindent Notice that there is \emph{overlap between their row-sets}. Also, notice that again the elements of the current attribute were sorted to facilitate the identification of all row-sets of the possible new biclusters.

In this scenario, since a bicluster $(I, J)$ is able to generate multiple descendants per attribute, with overlap between their row-sets, it is necessary to take some actions to avoid the generation of redundant biclusters (either enumerating the same bicluster more than once or mining non-maximal biclusters in their row-sets). In fact, these challenging issues can occur if two descendant biclusters share $minRow$ rows or more in their row-sets. For instance, assuming $minRow = 3$, in our example in Figure~\ref{fig:variosFilhos2}, biclusters (d1) and (d2) cannot generate redundant biclusters because they share only 2 rows in their row-sets. But biclusters (d2) and (d3) can generate redundant biclusters with row-set equal to or contained in

\begin{center}
$\{i_4, i_5, i_6, i_9, i_{10}, i_{12}\}.$
\end{center}

\noindent Let us suppose that there is a maximal CVC bicluster with row-set equal to $\{i_4, i_5, i_9, i_{10}\}$. So, we need to prevent this bicluster from being enumerated more than once, being a descendant of (d2) and (d3). Also, let us suppose that there is a maximal CVC bicluster with row-set equal to $\{i_4, i_5, i_9, i_{10}, i_{11}\}$. So, this bicluster must be a descendant of (d3), and we need to avoid mining a non-maximal bicluster with row-set equal to $\{i_4, i_5, i_9, i_{10}\}$ that would be a descendant of (d2).

To solve the first problem, i.e., enumerating the same bicluster more than once, the previous version of RIn-Close\_CVC \cite{VeronezeEtAl2017} tracks the row-sets that have already been generated using efficient symbol table implementations, such as hash tables (HTs) or balanced search trees (BSTs). The symbol table's keys are given by the row-sets, in such way that the rows in a row-set are in their ascending (or descending) order. Notice that it solves this first problem because two distinct CVC biclusters must have two distinct row-sets in order to be maximal. Thus, using a symbol table to track the row-sets, RIn-Close\_CVC \cite{VeronezeEtAl2017} does not mine the same bicluster more than once.

To solve the second problem, i.e., mining non-maximal biclusters in their row-sets, RIn-Close\_CVC \cite{VeronezeEtAl2017} verifies if the bicluster is row-maximal. Considering the Definition~\ref{def:maximal}, we see that it can be done by testing all rows in $X \setminus G$, where $G$ is the row-set of the possible new bicluster. However, it is possible to test a much smaller number of rows. For instance, bicluster (d3) in Figure~\ref{fig:variosFilhos2} should test only the rows corresponding to $i_1$ and $i_8$. We call $\Gamma$ the set of rows that must be checked to verify the row-maximality of the descendants of a bicluster. We will explain how to compute $\Gamma$ in what follows.

Algorithm~\ref{alg:rinclosecvc} presents the pseudocode of RIn-Close\_CVC \cite{VeronezeEtAl2017}. It is invoked with an initial pair $(I, J) = (X, \emptyset)$, an initial attribute $y = 1$, and an empty set $\Gamma = \emptyset$. The symbol table that tracks the biclusters that have already been mined is called $ST$ in the pseudocode. The user-defined parameters minimum number of rows $minRow$ and maximum perturbation $\epsilon$, and the numerical data matrix $\mathbf{A}_{n \times m}$ were omitted of the pseudocode for simplicity and clarity.

\begin{algorithm}
\caption{\textit{RIn-Close\_CVC}\\ComputeBiclustersFrom$((I,J),y,\Gamma)$}
\label{alg:rinclosecvc}
\begin{algorithmic}[1]
  \small
  \REQUIRE Bicluster $(I,J)$ to be closed, current attribute $y$, set of rows to check the row-maximality of the descendants $\Gamma$
  \FOR{$j \leftarrow y$ to $m$}
	  \IF{$j \notin J$}
		  \IF{$\max_{i \in I}(a_{ij}) - \min_{i \in I}(a_{ij}) \leq {\color{black}\epsilon_j}$}
			  \STATE $J \leftarrow J \cup \{j\}$
			\ELSE
			  \STATE Compute the possible new row-sets \COMMENT{Eq.~\ref{eq:rinc_cvc_compExt}}
			  \FOR{each possible new row-set $G$}
					\IF{$|G| \geq minRow$ \AND $G \notin ST$ \AND $G$ is canonical \AND $G$ is row-maximal}
					    \STATE Insert $G^s$ in the symbol table $ST$
						\STATE $\Omega \leftarrow ComputeRM(G, j, \Gamma, I)$ \COMMENT{Algorithm~\ref{alg:ComputeRM}}
						\STATE PutInQueue($G, j, \Omega$)
					\ENDIF					
				\ENDFOR
			\ENDIF
		\ENDIF
	\ENDFOR
    \STATE ProcessBicluster($I,J$)
	\WHILE{GetFromQueue($G, j, \Omega$)}
	  \STATE $H \leftarrow J \cup \{j\}$
		\STATE ComputeBiclustersFrom($(G,H),j+1, \Omega$)
	\ENDWHILE
\end{algorithmic}
\end{algorithm}

Although RIn-Close\_CVC is more elaborate than RIn-Close\_CVCP, the pseudocode has the same structure. So, each bicluster $(I,J)$ is incrementally closed, i.e., its column-set $J$ is completed with all possible columns for the row-set $I$. If the attribute $j$ is not an inherited attribute and it cannot be added to the column-set $J$, the possible new row-sets are computed (line 6). Given that $I$ is the current row-set, and $j$ is the current attribute, the possible new row-sets are given by
\begin{equation}
	\{G \: | \: [G \subseteq I] \; \wedge \; [\max_{i \in G}(\{a_{ij}\}) - \min_{i \in G}(\{a_{ij}\}) \leq \epsilon] \; \wedge \; [G \; \mathrm{is \; maximal}]\}.
	\label{eq:rinc_cvc_compExt}
\end{equation}

\noindent It is easily achieved by \textbf{sorting} the values of the data matrix $\mathbf{A}$ in rows $I$ and column $j$. After ordering, just scroll through the vector looking for the sets of values that meets the user-defined maximum perturbation $\epsilon$.


Letting $J$ be the current column-set, and $j$ be the current attribute, the row-set $G$ of a possible new bicluster is not canonical if

\begin{equation}
    \exists k \in Y \setminus J \: | \: [k < j] \: \wedge \: [\max_{i \in G}(a_{ik}) - \min_{i \in G}(a_{ik}) \leq \epsilon],
	\label{eq:rinc_cvc_iscan}
\end{equation} 

\noindent i.e., if there is an attribute $k < j$ that we can add to the bicluster $(G, J)$ and it remains a correct CVC bicluster. Also, letting $H = J \cup \{j\}$, and $\Gamma$ be the set of rows that must be checked to verify the row-maximality, the candidate new bicluster with row-set $G$ is not row-maximal if there is an object $g \in \Gamma$ that we can add to the bicluster $(G,H)$ and it remains a correct CVC bicluster, i.e.,

\begin{equation}
  \exists g \in \Gamma \: | \: \max_{i \in \{G \cup \{g\}\}}(a_{ik}) - \min_{i \in \{G \cup \{g\}\}}(a_{ik}) \leq \epsilon, \forall k \in H. 
	\label{eq:cvc_ismaximal}
\end{equation}

\begin{algorithm}
\caption{ComputeRM}
\label{alg:ComputeRM}
\begin{algorithmic}[1]
  \small
  \REQUIRE row-set $G$, attribute $j$, set of rows to check the row-maximality $\Gamma$, row-set $I$
  \ENSURE new set of rows to check the row-maximality $\Omega$
  \STATE $p1 \leftarrow minRow$-$th$ smaller element of $\{a_{ij}\}_{i \in G}$ \COMMENT{pivot value 1}
  \STATE $p2 \leftarrow minRow$-$th$ larger element of $\{a_{ij}\}_{i \in G}$ \COMMENT{pivot value 2}
  \STATE $\Omega \leftarrow \Gamma \cup \{i \in I \setminus G \; | \; [p1 - a_{ij} \leq \epsilon] \;  \vee \; [a _{ij} -p2 \leq \epsilon] \}$
\end{algorithmic}
\end{algorithm}

Now, let us explain the function $ComputeRM$ of Algorithm~\ref{alg:ComputeRM}. The pivot elements are the $minRow$-$th$ smaller and larger elements of $\{a_{ij}\}_{i \in G}$. To exemplify, let us suppose that $\{a_{ij}\}_{i \in G} = \{3, 3, 4, 4.5, 5, 6\}$, $minRow = 2$, and $\epsilon = 3$.  So, $p1 = 3$ and $p2 = 5$. Considering the current attribute $j$, rows $i \in I \setminus G$ with values greater than or equal to 0 ($p1 - \epsilon$) or less than or equal to 8 ($p2 + \epsilon$) must comprise $\Omega$. In addition, $\Omega$ inherits set of rows $\Gamma$. Let us explain the motivation behind this inheritance using the example of Figure~\ref{fig:variosFilhos2}: suppose that the bicluster (d2) gives rises to a new bicluster $(G,H)$ with row-set $G=\{i_8, i_4, i_6, i_{10}, i_{12}\}$, and there is a bicluster in the dataset with row-set $\{i_4, i_{10}, i_{12}, i_3\}$. Clearly, the bicluster $(G,H)$ must inherit from its parent the set of rows to check the row-maximality in order to not generate a non-row-maximal bicluster with row-set $\{i_4, i_{10}, i_{12}\}$.

\textbf{Worst-case complexity} - Unlike what happens in the enumeration of perfect CVC biclusters, it could have overlap between the row-sets of the candidate new biclusters. So, if there is 1 possible new row-set, it has at most $n$ rows; if there is 2 possible new row-sets, they have at most $n-1$ rows each; if there is 3 possible new row-sets, they have at most $n-2$ rows each; and so on. Therefore, we have a higher cost in $n$ in this algorithm. The worst-case time of checking if a candidate new bicluster is row-maximal is the same as the canonicity verification: $O(nm)$. The worst-case time of the function \textit{ComputeRM} is $O(n)$. The worst-case time to insert and search in a BST is $O(\log q)$, where $q$ is its number of elements. The worst-case time to insert and search in a HT is $O(1)$ and $O(q)$, respectively. However, under reasonable assumptions, the average time to search in a HT is $O(1)$. Thus, the overall worst-case time of RIn-Close\_CVC is $O(qnm(n \log n + nm + x))$,  where $x$ is the worst-case time of searching in the symbol table. Remark that usually $\log n \ll m$.

\begin{proposition}
RIn-Close\_CVC is an (1) efficient, (2) complete, (3) correct, and (4) non-redundant algorithm for mining all maximal CVC biclusters, with maximum perturbation $\epsilon$, from a numerical data matrix $\mathbf{A}_{n \times m}$.
\end{proposition}
\begin{proof}
RIn-Close\_CVC generalizes RIn-Close\_CVCP, so it is sufficient to show that the generalization steps keep these 4 properties:

\noindent (1) Efficiency: RIn-Close\_CVC has also at most polynomial time in the input size. If it is using a BST, it has time complexity quasi-linear, more specifically, linearithmic in the number of biclusters. If it is using a HT, its worst-case time is quadratic in the number of biclusters, however, under reasonable assumptions, it will have linear time in the number of biclusters.

\noindent (2) Completeness: RIn-Close\_CVC also has the same search engine as In-Close2. So, the explanation is basically the same as that of RIn-Close\_CVCP. The main difference is the fact that the row-sets of the candidate new biclusters are computed according to Eq.~\ref{eq:rinc_cvc_compExt}, since now we are looking for perturbed CVC biclusters.

The row-maximal checking will only discard non-maximal biclusters in their row-sets. Also, since each maximal CVC bicluster has a unique row-set, the usage of the symbol table $ST$ will only avoid the mining of the same maximal bicluster more than once. So, these features of RIn-Close\_CVC do not interfere in the completeness of the algorithm.

\noindent (3) Correctness: The explanation is also the same as that of RIn-Close\_CVCP, except by the fact that now we are looking for perturbed CVC biclusters.

\noindent (4) Non-Redundancy:
As in In-Close2, the column-set of a bicluster is maximal after its closure. However, we must check if a candidate new bicluster with row-set $G$ is row-maximal. We could do this by testing all rows in $X \setminus G$, but it is sufficient to test the rows in the set $\Gamma$. Why? The rows in $\Gamma$ are computed taking into account that each bicluster must have at least $minRow$ rows, and all its columns must have at most perturbation $\epsilon$. Thus, the attribute $j$ of the new bicluster, where it was generated, is used as reference to compute a set of rows that must belong to the set $\Gamma$: based on the $minRow$-$th$ smaller and $minRow$-$th$ larger values of the column $j$ of the new bicluster, rows that could generated a correct CVC bicluster are added to $\Gamma$. Furthermore, the set $\Gamma$ of a new bicluster inherits the set of rows to check the row-maximality from all its ancestors. With this inheritance, all rows that could possibly generate a correct CVC bicluster are kept tracked. 

In addition to the inherited canonicity test of In-Close2, RIn-Close\_CVC also uses a symbol table for the prevention of mining the same (maximal) bicluster more than once. The keys of the symbol table are the ordered row-sets of the mined biclusters. Thus, as each maximal CVC bicluster has a unique row-set, RIn-Close\_CVC will not enumerate the same bicluster more than once.
\end{proof}

\section{New RIn-Close\_CVC}
\label{sec:rinclosecvc2}

We saw in Subsection~\ref{subsec:rinclosecvc} that during the closure of a CVC bicluster $(I,J)$, it can generate multiple descendants per attribute with overlap between their row-sets. We also saw that without taking some extra care in this scenario, an enumerative algorithm based on the search engine of In-Close2 could return redundant biclusters (either enumerating the same bicluster more than once or mining non-maximal biclusters in their row-sets). For ease, we will call this problem as \textit{row-sets overlap problem}.


The difference between the previous version of RIn-Close\_CVC \cite{VeronezeEtAl2017} and this new version, RIn-Close\_CVC2, is basically in the actions to deal with the \textit{row-sets overlap problem}.

Algorithm~\ref{alg:rinclosecvc2} shows the pseudocode of RIn-Close\_CVC2. Note that, unlike RIn-Close\_CVC \cite{VeronezeEtAl2017}, it does not have a symbol table $ST$ and neither the checking of row-maximality. Instead, it has a checking that we call \textit{row-canonical} to deal with the \textit{row-sets overlap problem}.

\begin{algorithm}
\caption{\textit{RIn-Close\_CVC2}\\ComputeBiclustersFrom$((I,J),y,\Gamma)$}
\label{alg:rinclosecvc2}
\begin{algorithmic}[1]
  \small
  \REQUIRE Bicluster $(I,J)$ to be closed, current attribute $y$, set of rows to check the row-canonicity of the descendants $\Gamma$
  \FOR{$j \leftarrow y$ to $m$}
	  \IF{$j \notin J$}
		  \IF{$\max_{i \in I}(a_{ij}) - \min_{i \in I}(a_{ij}) \leq {\color{black}\epsilon_j}$}
			  \STATE $J \leftarrow J \cup \{j\}$
			\ELSE
			  \STATE Compute the possible new row-sets \COMMENT{Eq.~\ref{eq:rinc_cvc_compExt}}
			  \FOR{each possible new row-set $G$}
					\IF{$|G| \geq minRow$ \AND $G$ is canonical \AND $G$ is row-canonical}
						\STATE $\Omega \leftarrow ComputeRM(G, j, \Gamma, I)$ \COMMENT{Algorithm~\ref{alg:ComputeRM}}
						\STATE PutInQueue($G, j, \Omega$)
					\ENDIF					
				\ENDFOR
			\ENDIF
		\ENDIF
	\ENDFOR
    \STATE ProcessBicluster($I,J$)
	\WHILE{GetFromQueue($G, j, \Omega$)}
	  \STATE $H \leftarrow J \cup \{j\}$
		\STATE ComputeBiclustersFrom($(G,H),j+1, \Omega$)
	\ENDWHILE
\end{algorithmic}
\end{algorithm}

Letting $J$ be the current column-set, $j$ the current attribute, $J^{<j}$ the set of all attributes of $J$ up to $j$, $H = J^{<j} \cup \{j\}$, and $\Gamma$ the set of rows that must be checked to verify the row-canonicity, the candidate new bicluster with row-set $G$ is not row-canonical if (1) there is an object $g \in \Gamma$ that we can add to the bicluster $(G,H)$ so that it remains a correct CVC bicluster, i.e.,

\begin{equation}
  \exists g \in \Gamma \: | \: \max_{i \in \{G \cup \{g\}\}}(a_{ik}) - \min_{i \in \{G \cup \{g\}\}}(a_{ik}) \leq \epsilon, \forall k \in H,
	\label{eq:cvc_row-iscan1}
\end{equation}

\noindent or (2) there is another bicluster with minor lexicographic order in the rows that could be the parent of the candidate new bicluster, i.e.,

\begin{equation}
  \exists g \in \Gamma \: | \: \max_{i \in \{I^{<g} \cup \{g\} \cup G\}}(a_{ik}) - \min_{i \in \{I^{<g} \cup \{g\} \cup G\}}(a_{ik}) \leq \epsilon, \forall k \in J^{<j},
	\label{eq:cvc_row-iscan2}
\end{equation}

\noindent where $I^{<g}$ is the set of all objects of $I$ up to $g$ ($g \notin I$ by definition).

The idea behind the function row-canonical is as follows. If a bicluster can be created from more than one bicluster on different columns, it must be created in the most posterior column. If a bicluster can be created from more than one bicluster in the same column $j$, it must be created by the bicluster with the least lexicographic order in its row-set. Besides, the first part of the function row-canonical also ensures that a new bicluster is row-maximal.

The set of rows that must be checked to verify the row-canonicity $\Gamma$ is also computed by Algorithm~\ref{alg:ComputeRM}, since all rows that could possibly generate a correct CVC bicluster, due to the \textit{row-sets overlap problem}, are kept tracked by means of it.

\textbf{Worst-case complexity} - The worst-case time of checking if a candidate new bicluster is row-canonical is $O(n^2m)$. Thus, the overall worst-case time of RIn-Close\_CVC2 is $O(qn^3m^2)$.

\begin{proposition}
RIn-Close\_CVC2 is an (1) efficient, (2) complete, (3) correct, and (4) non-redundant algorithm for mining all maximal CVC biclusters, with maximum perturbation $\epsilon$, from a numerical data matrix $\mathbf{A}_{n \times m}$.
\end{proposition}
\begin{proof}
The difference between RIn-Close\_CVC and RIn-Close\_CVC2 is only in the way of dealing with the \textit{row-sets overlap problem}. The new way of dealing with this issue is the row-canonicity checking (see Eqs.~\ref{eq:cvc_row-iscan1} and \ref{eq:cvc_row-iscan2}). So, we only need to show that this alternative step keeps these 4 properties.

\noindent (1) Efficiency: RIn-Close\_CVC2 has time complexity linear in the number of biclusters and polynomial in the input size.

\noindent (2) Completeness: The row-canonicity will not lose a (maximal) bicluster because a new candidate bicluster will only be discarded if (1) there is a bicluster $(G \cup \{g\}, J^{<j} \cup \{j\})$, which may create it in a column greater than $j$ ($g \in \Gamma$), or (2) there is another bicluster with minor lexicographic order in its row-set, which may create it in the column $j$.

\noindent (3) Correctness: Idem RIn-Close\_CVC.

\noindent (4) Non-Redundancy:
We have already shown that the set $\Gamma$ contains all rows that could possibly generate a correct CVC bicluster due to the \textit{row-sets overlap problem}.

The first part of the row-canonicity test (see Eq.~\ref{eq:cvc_row-iscan1}) ensures that a candidate new bicluster that is not row-maximal will be discarded.

In the case of mining the same bicluster more than once, we can have two situations: a bicluster can be created by more than one bicluster (1) on different columns, and/or (2) in a same column $j$. Both situations are covered by the new row-canonicity test. Its first part ensures that a new bicluster will be created in the most posterior column by checking if there exists a correct CVC bicluster $(G \cup \{g\}, J^{<j} \cup \{j\})$. The second part ensures that a new bicluster will only be created in a column $j$ by the bicluster with the minor lexicographical order in its row-set. Therefore, the test fails if there is a correct CVC bicluster with row-set $\{I^{<g} \cup \{g\} \cup G\} < I$ and column-set $J^{<j}$ (remark that $G \subset I$ and if $I^{<g}  = I$, then the bicluster $(I^{<g} \cup \{g\} \cup G, J^{<j})$ is incorrect by definition since $I$ is a maximal row-set).
\end{proof}

\section{Experimental Results}
\label{sec:expresults}

We evaluated the new version of RIn-Close\_CVC, denoted RIn-Close\_CVC2, on both synthetic and real datasets. Our goal is to outline the advantages of this new version when compared to its previous version. The experiments were carried out on a PC Intel(R) Core(TM) i7-4770K CPU @ 3.5 GHz, 32GB of RAM, and running under Ubuntu 14.04. Our code is available at \url{https://sourceforge.net/projects/rinclose/}.

\subsection{Artificial Data}

This experiment aims to test RIn-Close\_CVC2's performance when varying ($i$) the number $n$ of rows of the dataset, ($ii$) the number $m$ of columns of the dataset, ($iii$) the number of biclusters in the dataset, ($iv$) the bicluster row size, ($v$) the bicluster column size, and ($vi$) the overlap among the biclusters. For this purpose, we created synthetic datasets with controlled number, size, shape and level of noise of the existing biclusters and, then, we tested how RIn-Close\_CVC2 performs when varying each one of the parameters in isolation.

The default parameters used in the synthetic data generator were: $n = 10,000$, $m = 100$, number of biclusters $= 30$, bicluster row size $= 200$, bicluster column size $= 16$, overlap $= 0.2$, and Gaussian noise with $\mu = 0$ and $\sigma = 0.05$. The synthetic data generator creates the biclusters and assigns random values to the other regions of the dataset. Then, it adds Gaussian noise and shuffles the rows and columns of the dataset. Therefore, the generator creates arbitrarily positioned overlapping biclusters, so that the resulting biclusters are usually non-contiguous. The amount of noise was chosen in such a way that the original biclusters were preserved. The synthetic data generator returns a dataset and the maximum perturbation $\epsilon$ in its biclusters.

For each configuration, we created 50 different synthetic datasets to compute the median runtimes and memory usage. We chose the median, not the mean, because the median is less sensitive to outliers. For the same configuration, the results may vary greatly because the positioning of the biclusters is random. From FCA, we know that the placement of the biclusters in a dataset makes all difference in the runtime because the less canonicity test failures, the faster the execution \cite{CarpinetoEtAl2004}.

Notice that RIn-Close\_CVC and RIn-Close\_CVC2 produce exactly the same biclustering solution, finding all the planted biclusters without loosing any row or column.

Figures~\ref{fig:expSynDataRT} and \ref{fig:expSynDataMEM} shows the runtime and the memory usage of RIn-Close\_CVCP, RIn-Close\_CVC, and RIn-Close\_CVC2 for the different configurations of this experiment. We included RIn-Close\_CVCP as a baseline for comparison. Since RIn-Close\_CVCP looks for perfect biclusters, it was applied to datasets without the Gaussian noise.

For the runtime, RIn-Close\_CVC and RIn-Close\_CVC2 showed the same behavior along the variation of the synthetic data generator parameters. However, although RIn-Close\_CVC2 has a higher worst-case time-complexity than RIn-Close\_CVC, it obtained a better median runtime than RIn-Close\_CVC in all scenarios.

The memory usage of RIn-Close\_CVC2 was equivalent to the memory usage of RIn-Close\_CVCP, being much better than the one of RIn-Close\_CVC.

\begin{figure*}
\centering
\subfigure[]{
  \includegraphics[trim=0.2cm 0.1cm 0.6cm 0.4cm, clip, scale=0.3]{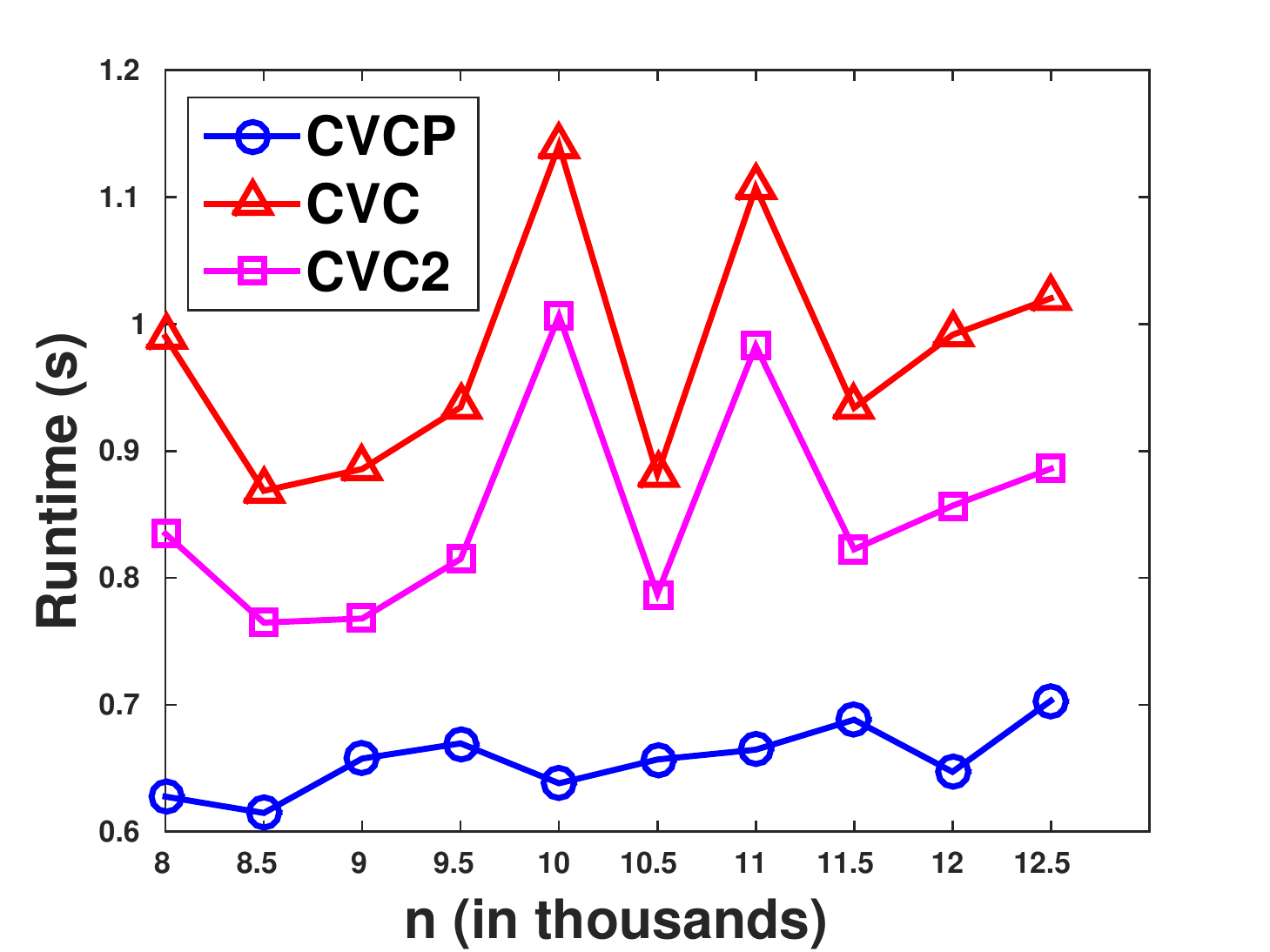}
}
\subfigure[]{
  \includegraphics[trim=0.2cm 0.1cm 0.6cm 0.4cm, clip, scale=0.3]{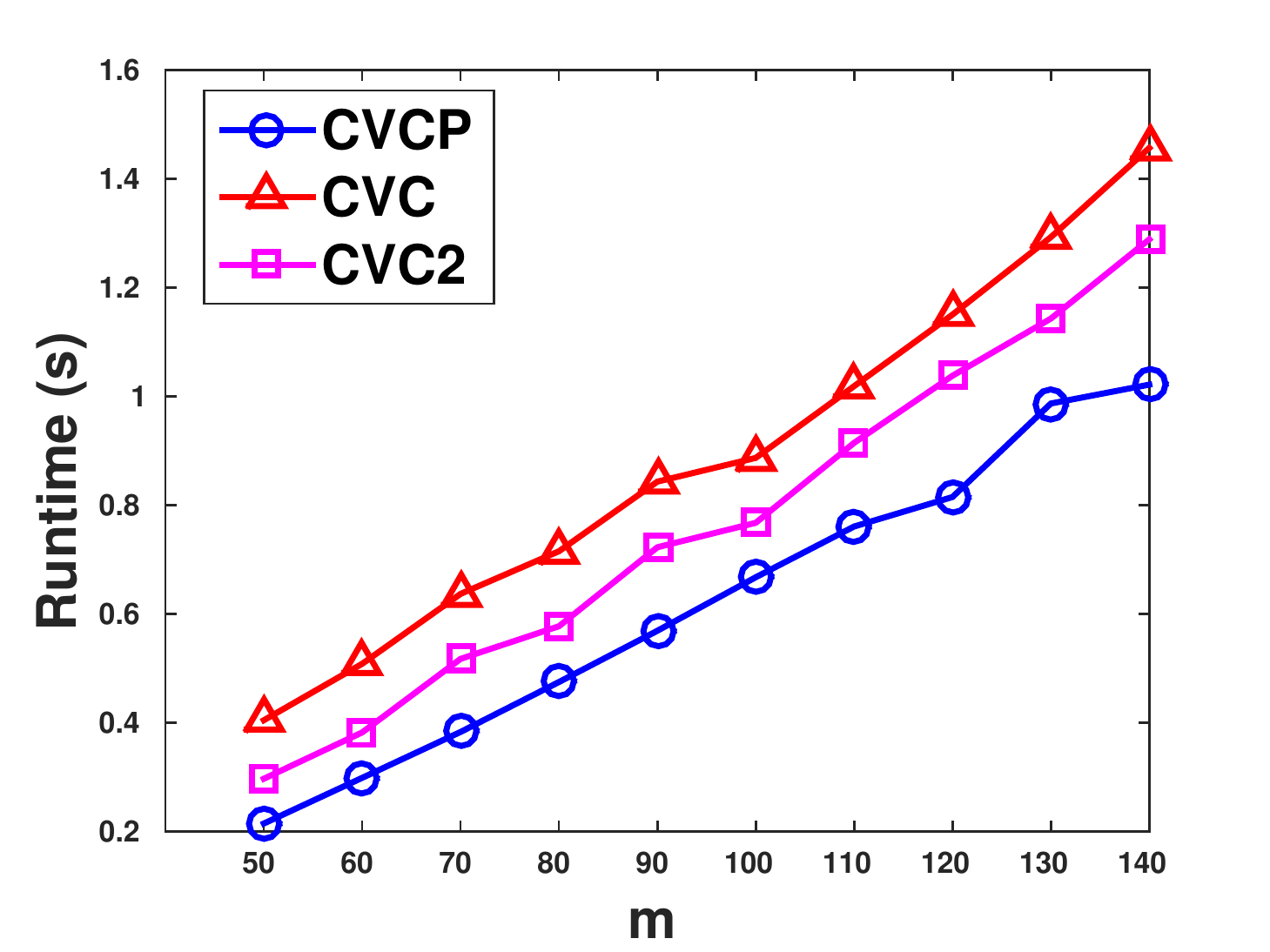}
}
\subfigure[]{
  \includegraphics[trim=0.2cm 0.1cm 0.6cm 0.4cm, clip, scale=0.3]{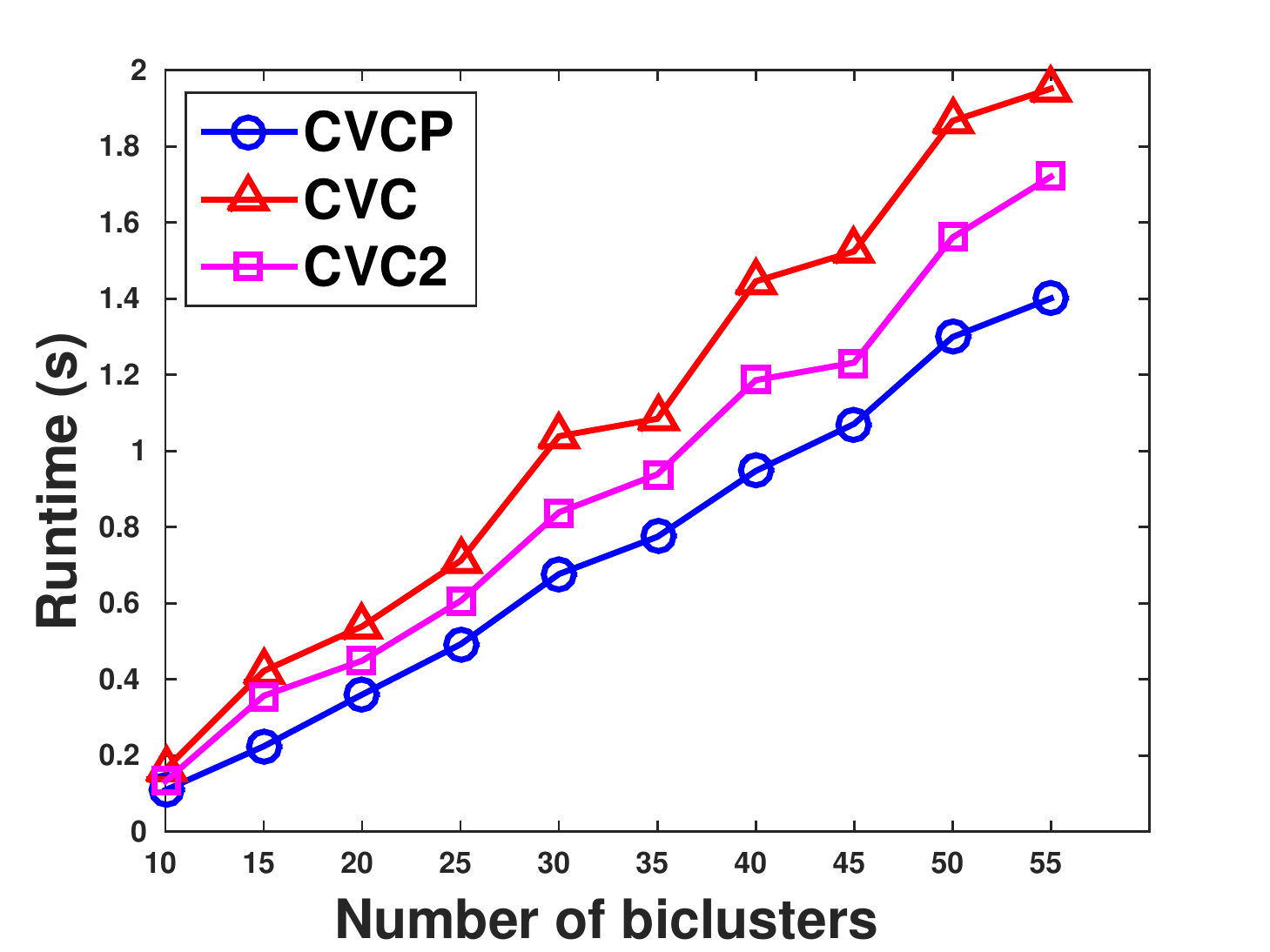}
}
\subfigure[]{
  \includegraphics[trim=0.2cm 0.1cm 0.6cm 0.4cm, clip, scale=0.3]{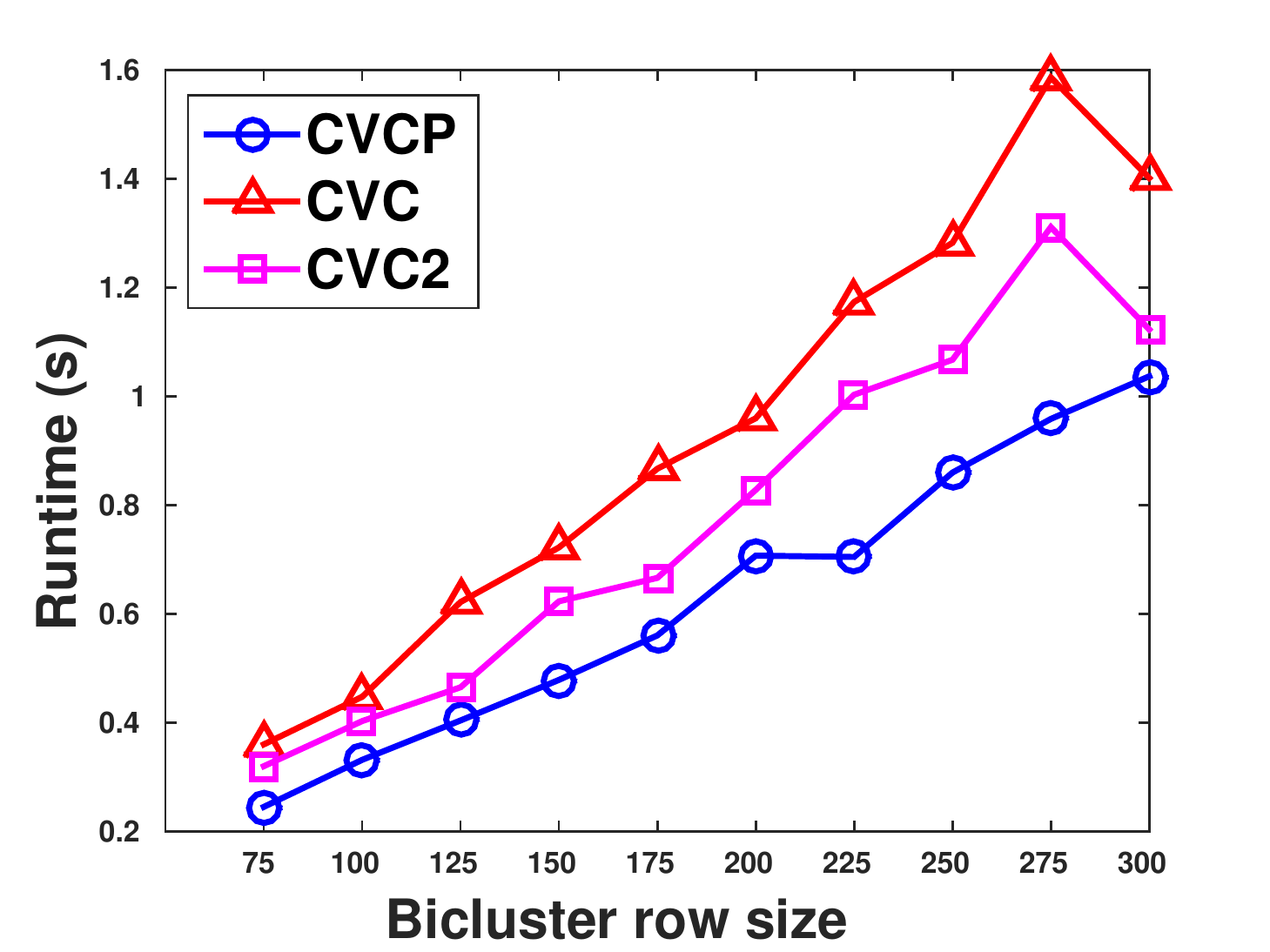}
}
\subfigure[]{
  \includegraphics[trim=0.2cm 0.1cm 0.6cm 0.4cm, clip, scale=0.3]{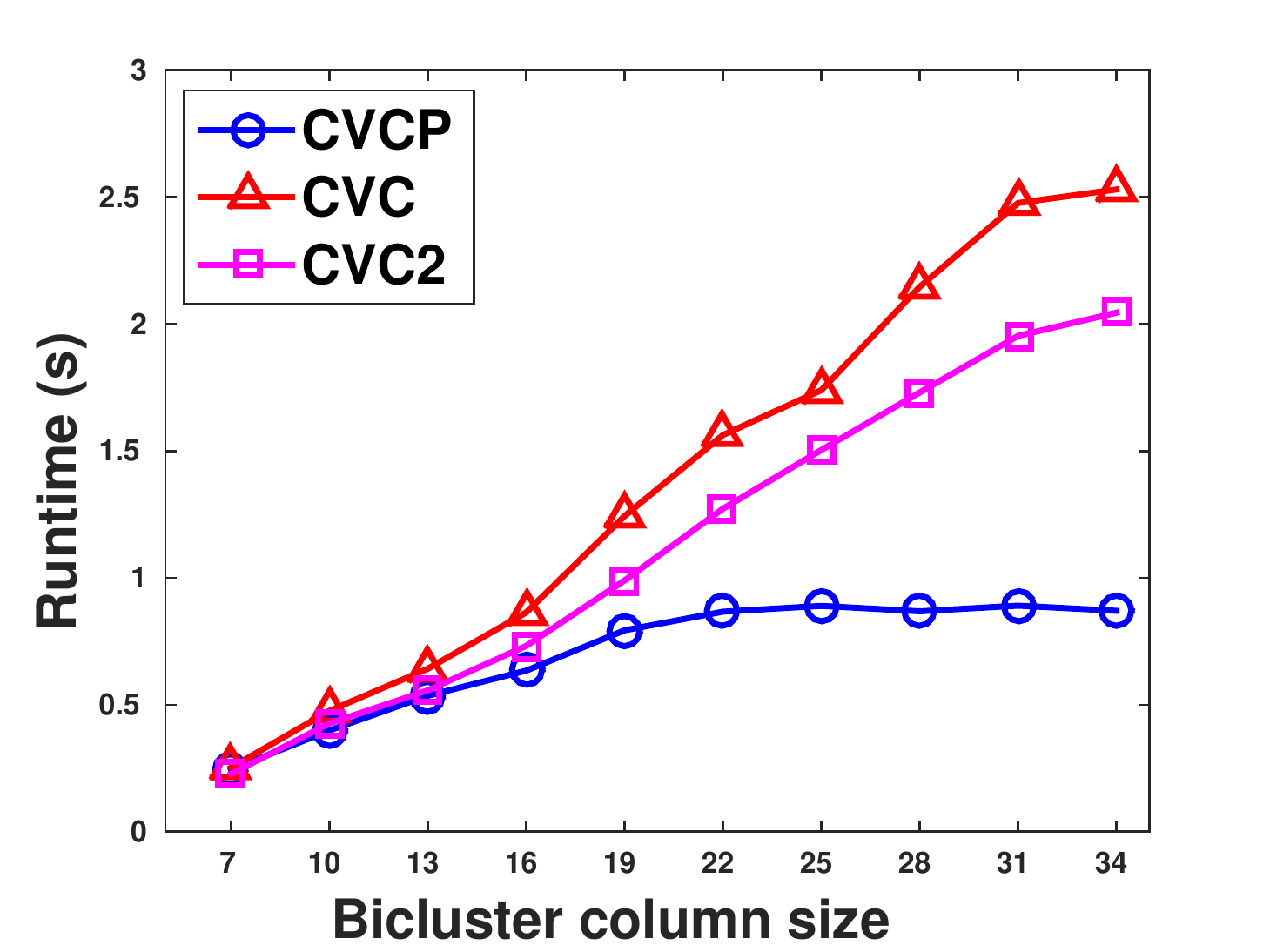}
}
\subfigure[]{
  \includegraphics[trim=0.2cm 0.1cm 0.6cm 0.4cm, clip, scale=0.3]{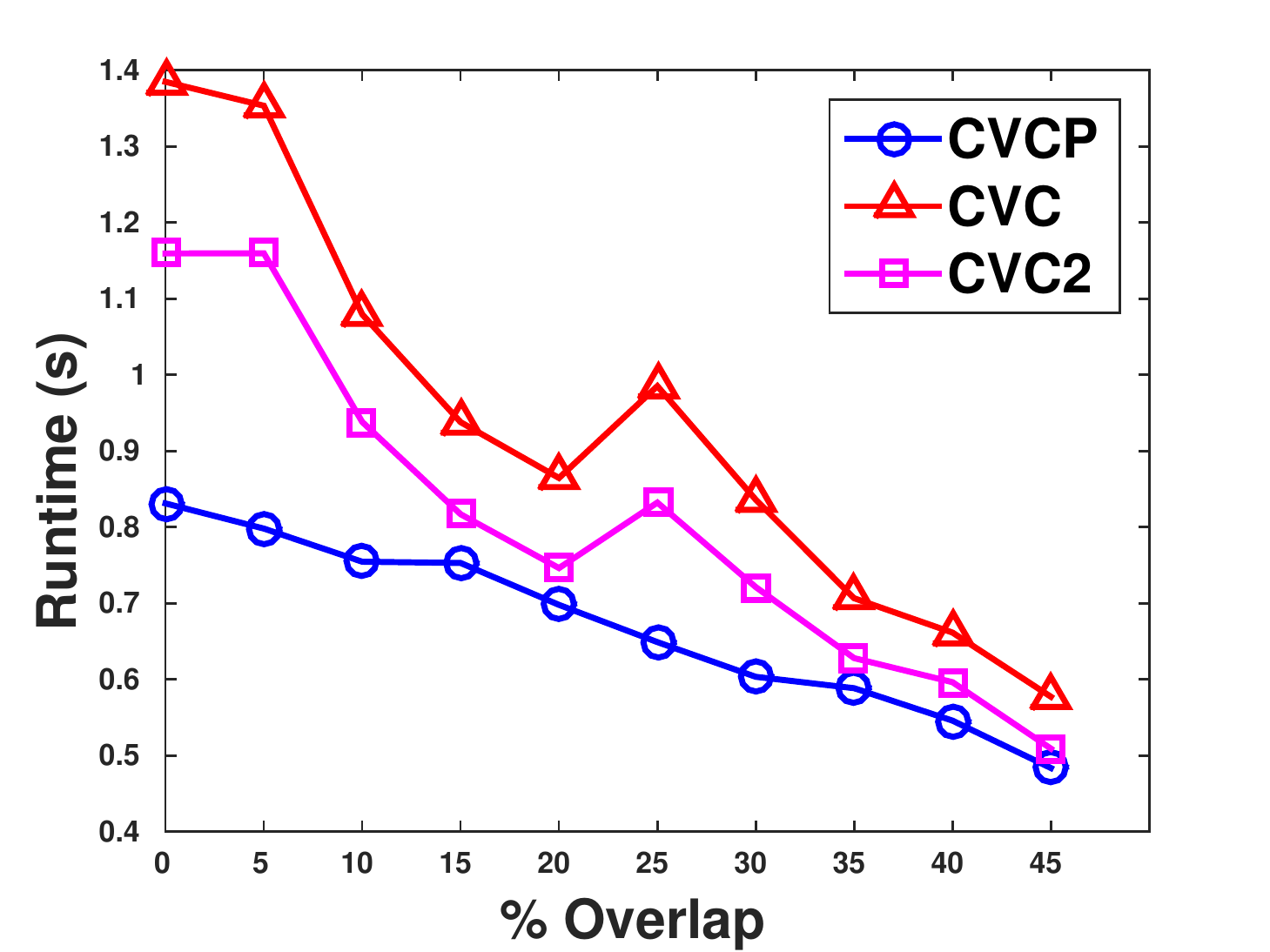}
}
\caption{Runtime of RIn-Close\_CVCP, RIn-Close\_CVC, and RIn-Close\_CVC2 when varying (a) the number $n$ of rows of the dataset, (b) the number $m$ of columns of the dataset, (c) the number of biclusters in the dataset, (d) the bicluster row size, (e) the bicluster column size, and (f) the overlap among the biclusters.}
\label{fig:expSynDataRT}
\end{figure*}

\begin{figure*}
\centering
\subfigure[]{
  \includegraphics[trim=0.2cm 0.1cm 0.6cm 0.4cm, clip, scale=0.3]{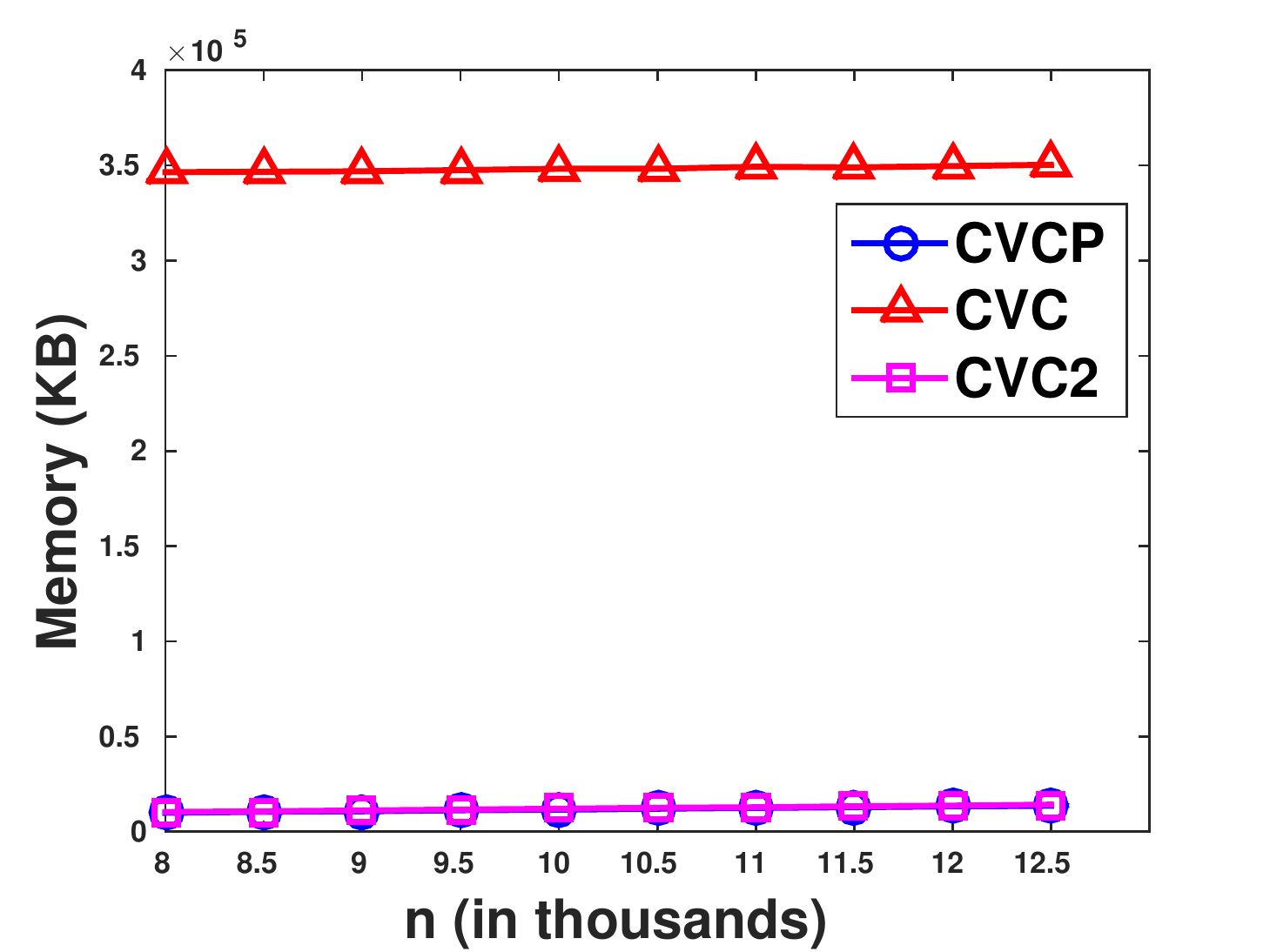}
}
\subfigure[]{
  \includegraphics[trim=0.2cm 0.1cm 0.6cm 0.4cm, clip, scale=0.3]{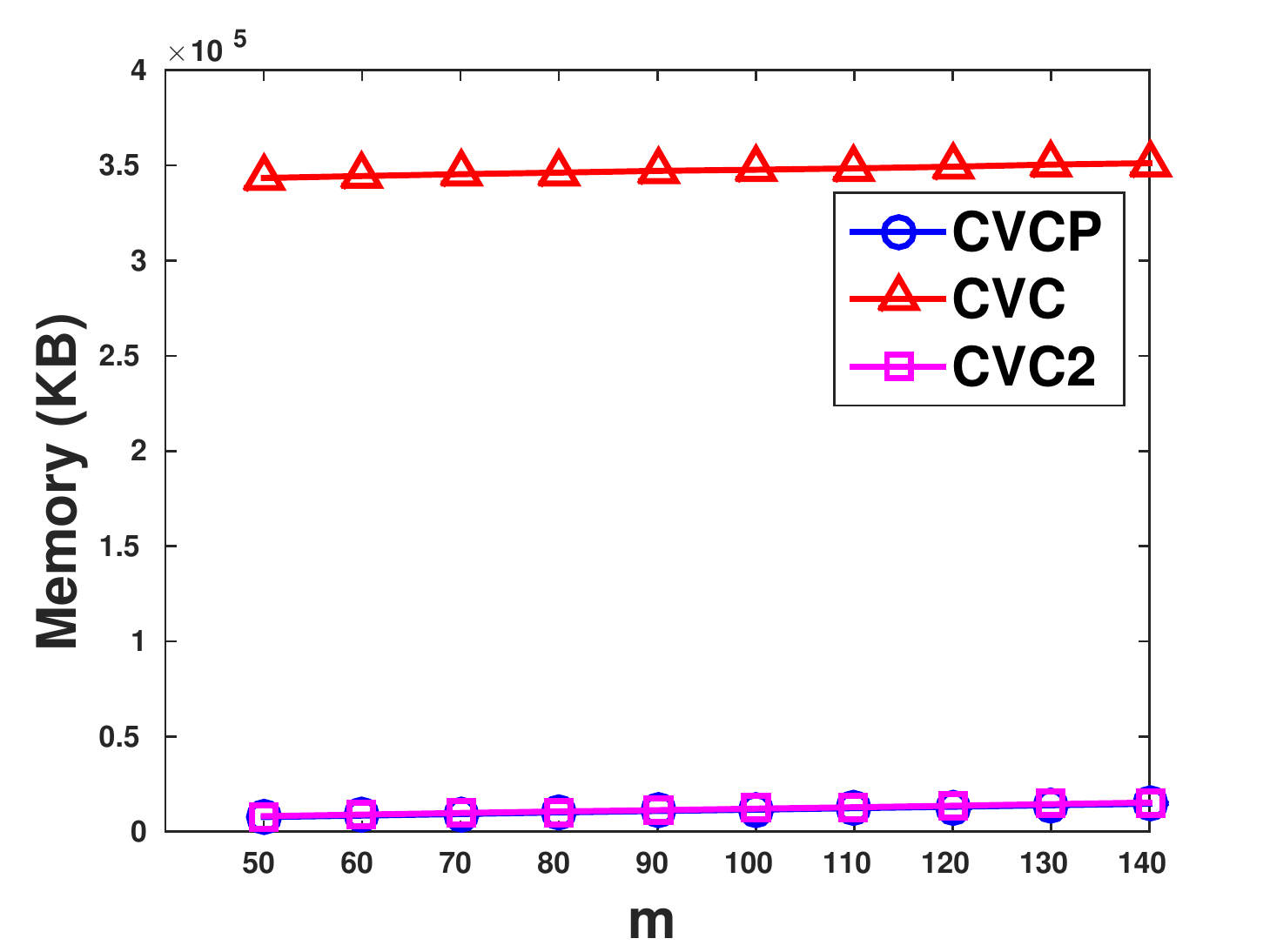}
}
\subfigure[]{
  \includegraphics[trim=0.2cm 0.1cm 0.6cm 0.4cm, clip, scale=0.3]{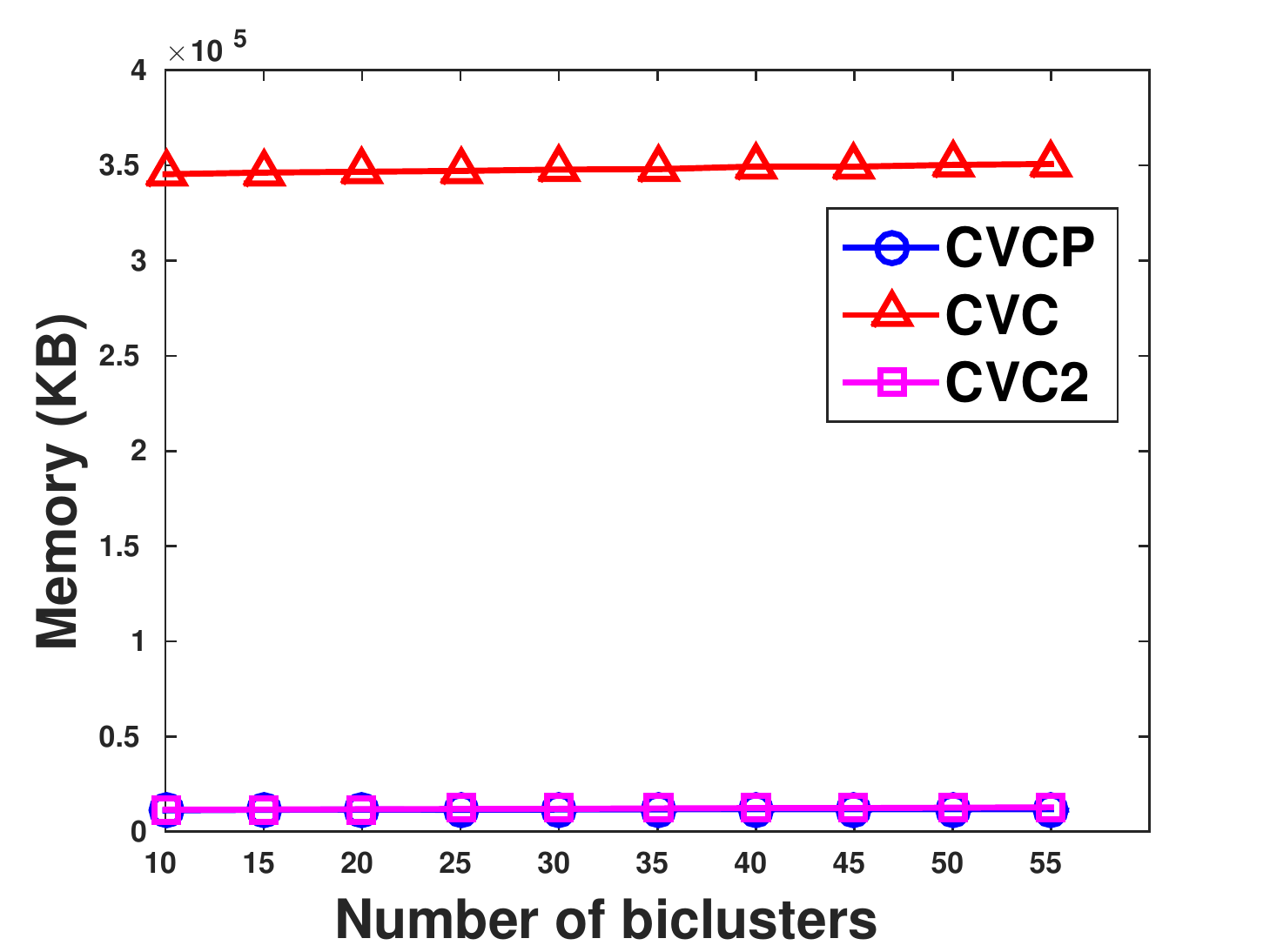}
}
\subfigure[]{
  \includegraphics[trim=0.2cm 0.1cm 0.6cm 0.4cm, clip, scale=0.3]{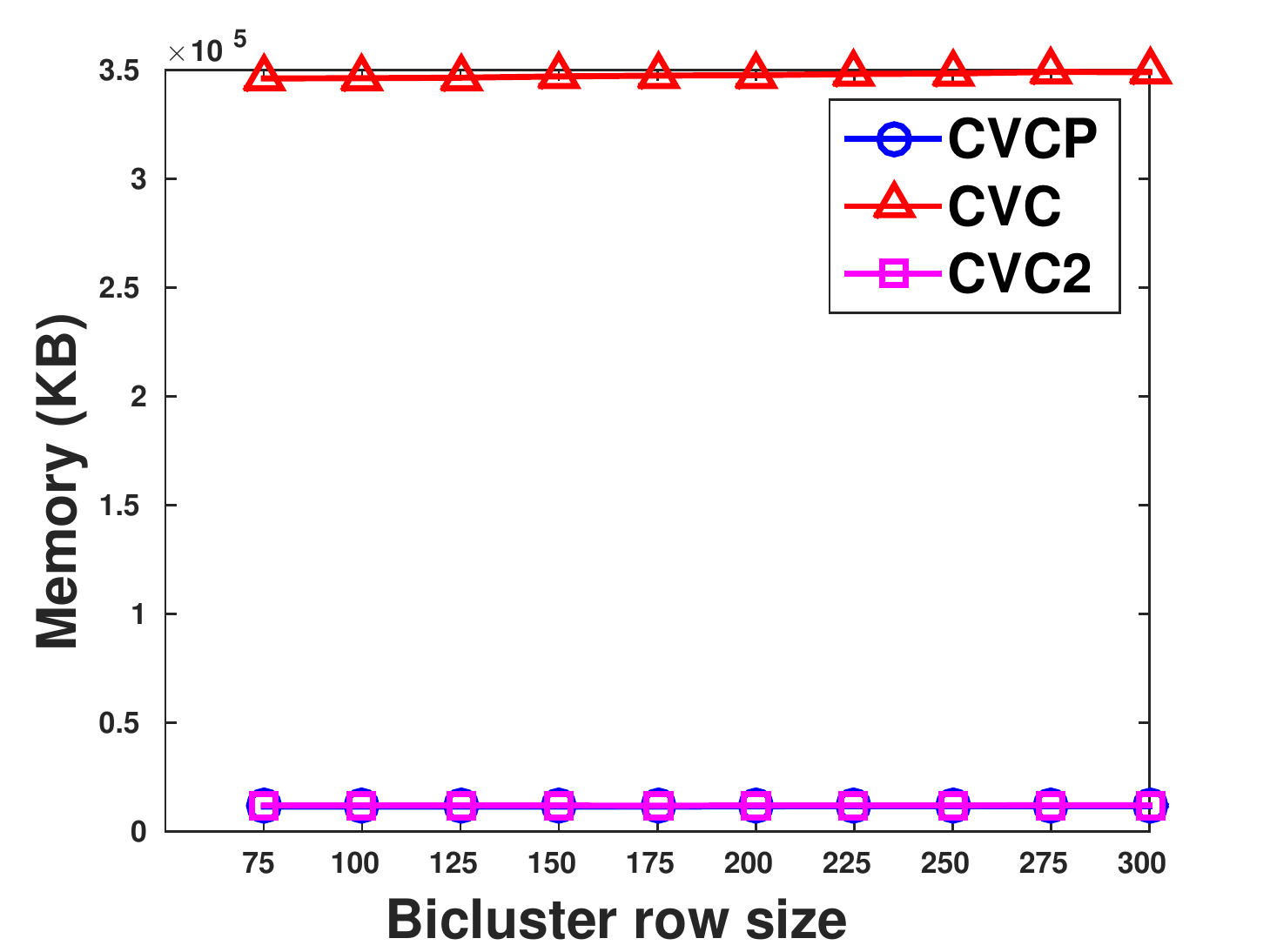}
}
\subfigure[]{
  \includegraphics[trim=0.2cm 0.1cm 0.6cm 0.4cm, clip, scale=0.3]{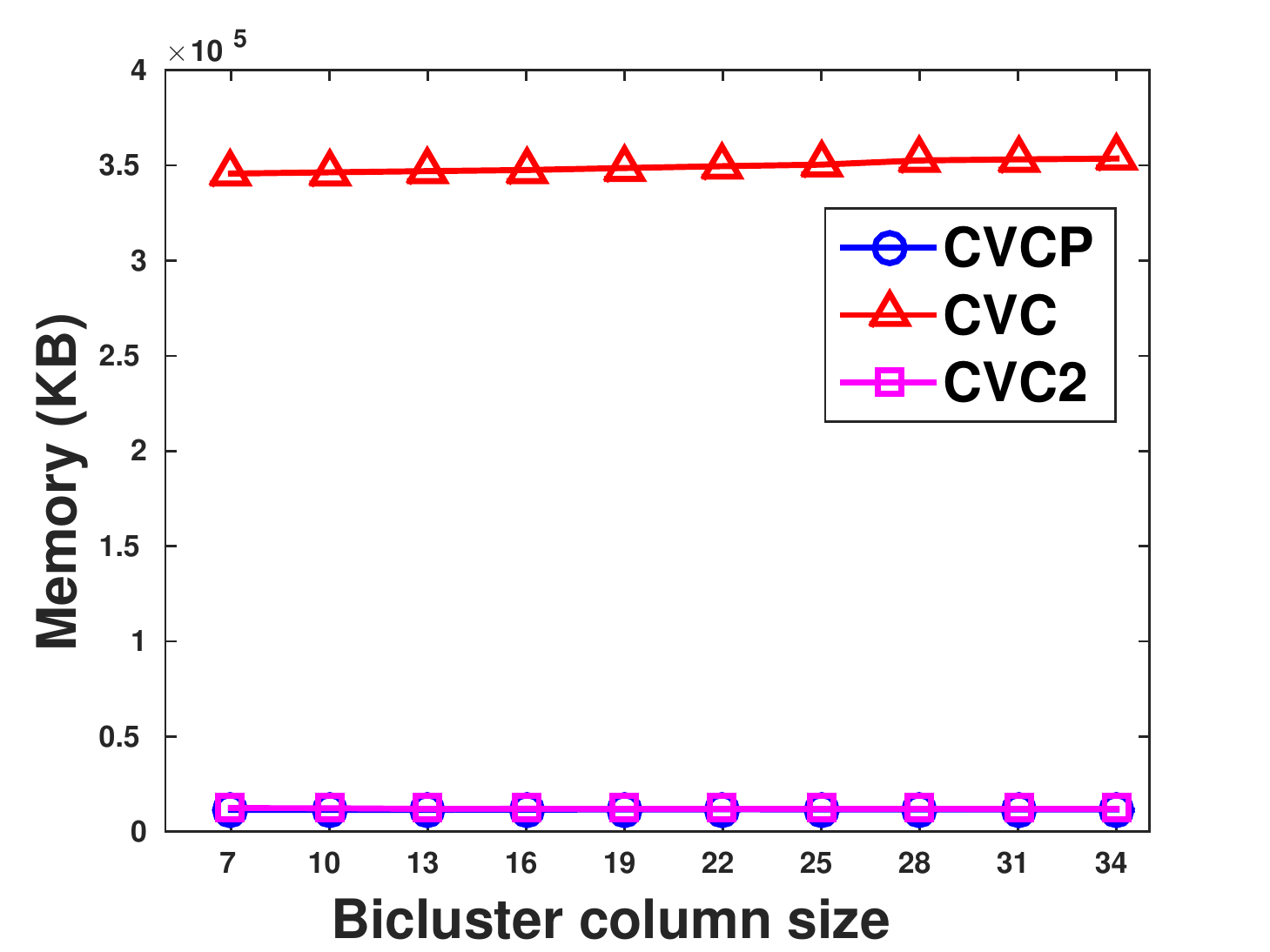}
}
\subfigure[]{
  \includegraphics[trim=0.2cm 0.1cm 0.6cm 0.4cm, clip, scale=0.3]{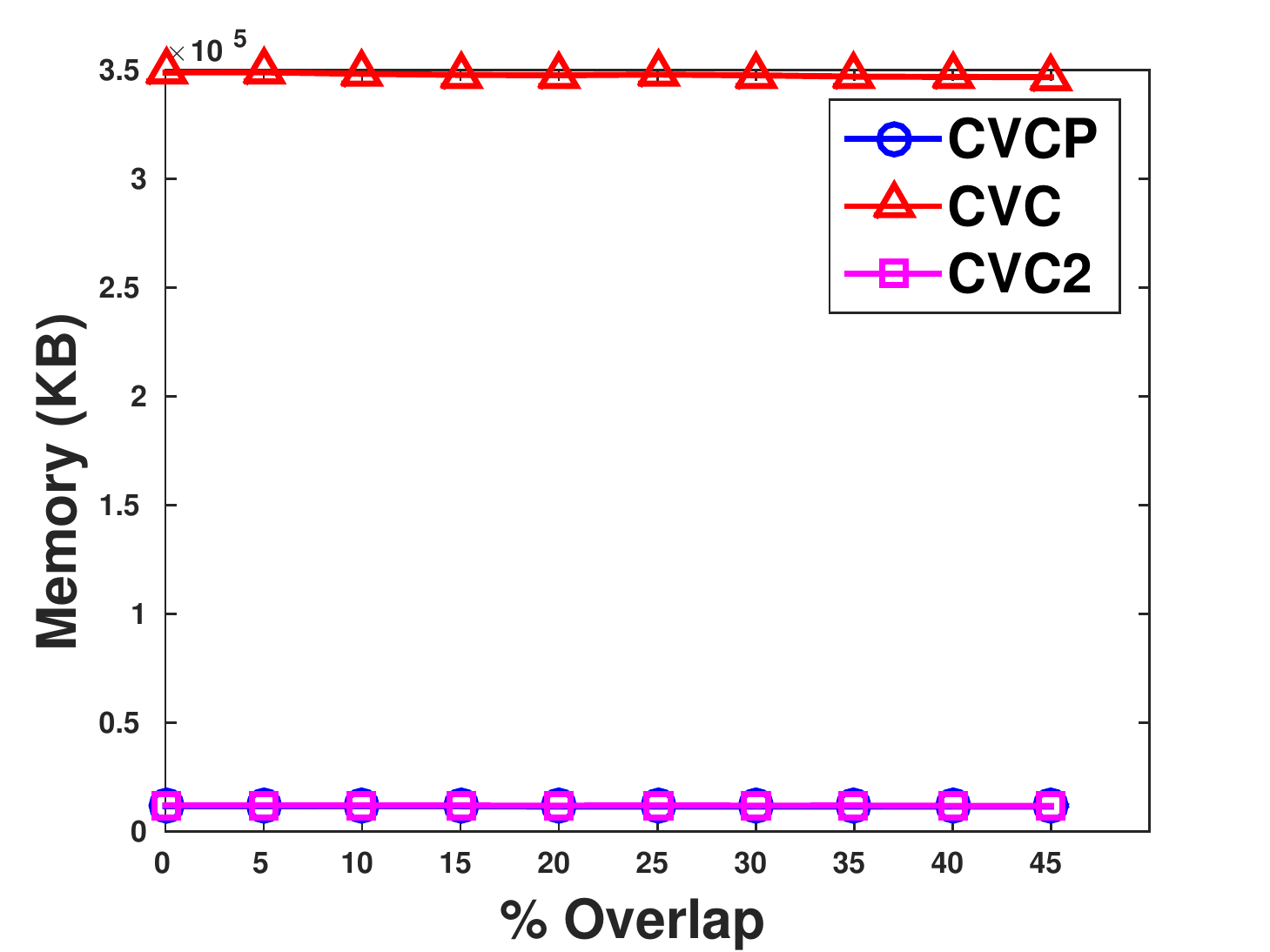}
}
\caption{Memory usage of RIn-Close\_CVCP, RIn-Close\_CVC, and RIn-Close\_CVC2 when varying (a) the number $n$ of rows of the dataset, (b) the number $m$ of columns of the dataset, (c) the number of biclusters in the dataset, (d) the bicluster row size, (e) the bicluster column size, and (f) the overlap among the biclusters.}
\label{fig:expSynDataMEM}
\end{figure*}

\subsection{Real Data}

This experiment aims to test the runtime and memory usage of RIn-Close\_CVC and RIn-Close\_CVC2 when real-world datasets are considered. For this, we used five gene expression datasets, and we varied the user-defined parameter $\epsilon$. This parameter defines the maximum perturbation allowed in the biclusters and, therefore, it is related with the number of biclusters that will be found in a dataset.

Table~\ref{tab:realdatasets} shows the main properties of the five real-world datasets that we used in our experiments. All of them was downloaded from the NCBI repository \footnote{\url{https://www.ncbi.nlm.nih.gov/}}. The attributes of the GDS750, GDS1981, GDS2267 and GDS3035 datasets have a skewed distribution, so we took the logarithm of the values of these datasets. Before taking the logarithm of the values of the GDS2267 dataset, we added a small constant ($1^{-100}$) to its values in order to avoid computations of $\log(0)$. We also scaled the data of each column to real-values between 0 and 1, which enables the usage of the same value of $\epsilon$ for all attributes.

We ran RIn-Close\_CVC and RIn-Close\_CVC2 for 50 times to compute the average runtime and memory usage. We looked for biclusters with at least 3 columns, and Table~\ref{tab:realdatasets} shows the minimum number of rows used for each dataset.

\begin{table}
\centering
\caption{Real-world datasets description.}
 \begin{tabular}{lrcr}
 \toprule
 Name	& Dimension	& Reference & $minRow$ \\
 \midrule
 GDS750	 & $6091 \times 13$ & \cite{LeberEtAl2004}  & 305 \\
 GDS759  & $6350 \times 24$ & \cite{SapraEtAl2004}  & 309 \\
 GDS1981 & $6178 \times 20$ & \cite{CarterEtAl2006} & 32  \\
 GDS2267 & $9335 \times 36$ & \cite{TuEtAl2005}     & 467 \\
 GDS3035 & $9335 \times 48$ & \cite{ShaEtAl2013}    & 467 \\
 \bottomrule
 \end{tabular}
\label{tab:realdatasets}
\end{table}

Figures~\ref{fig:expRealDataNBIC}, \ref{fig:expRealDataRT} and \ref{fig:expRealDataMEM} show, respectively, the number of biclusters, runtime and the memory usage of RIn-Close\_CVC and RIn-Close\_CVC2 for this experiment.

Logically, the number of biclusters is the same for RIn-Close\_CVC and RIn-Close\_CVC2, since both of them enumerate all maximal CVC biclusters. The number of biclusters increased exponentially with the increase of the value of $\epsilon$ for all datasets. For both algorithms, the runtime is proportional to the number of biclusters and follows the same pattern of growth of the number of biclusters. Again, RIn-Close\_CVC2 obtained a better runtime than RIn-Close\_CVC in all scenarios. The gain of RIn-Close\_CVC2 in the memory usage is enormous. It presents a linear growth even though the number of biclusters exhibits an exponential growth with the value of $\epsilon$, which does not happen in the previous version.

\begin{figure*}
\centering
\subfigure[]{
  \includegraphics[trim=0.4cm 0.1cm 1.4cm 0.15cm, clip, scale=0.35]{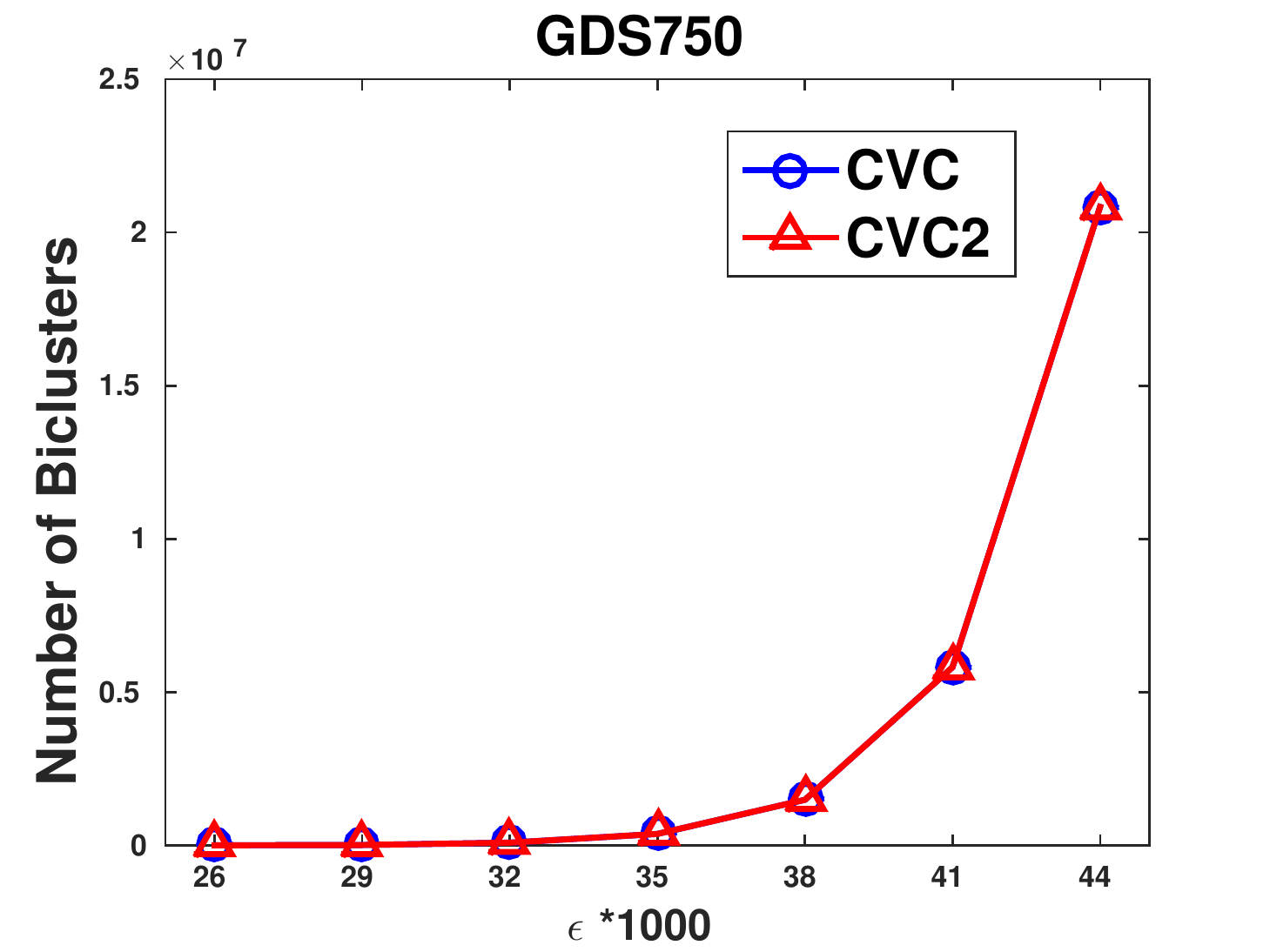}
}
\subfigure[]{
  \includegraphics[trim=0.4cm 0.1cm 1.4cm 0.15cm, clip, scale=0.35]{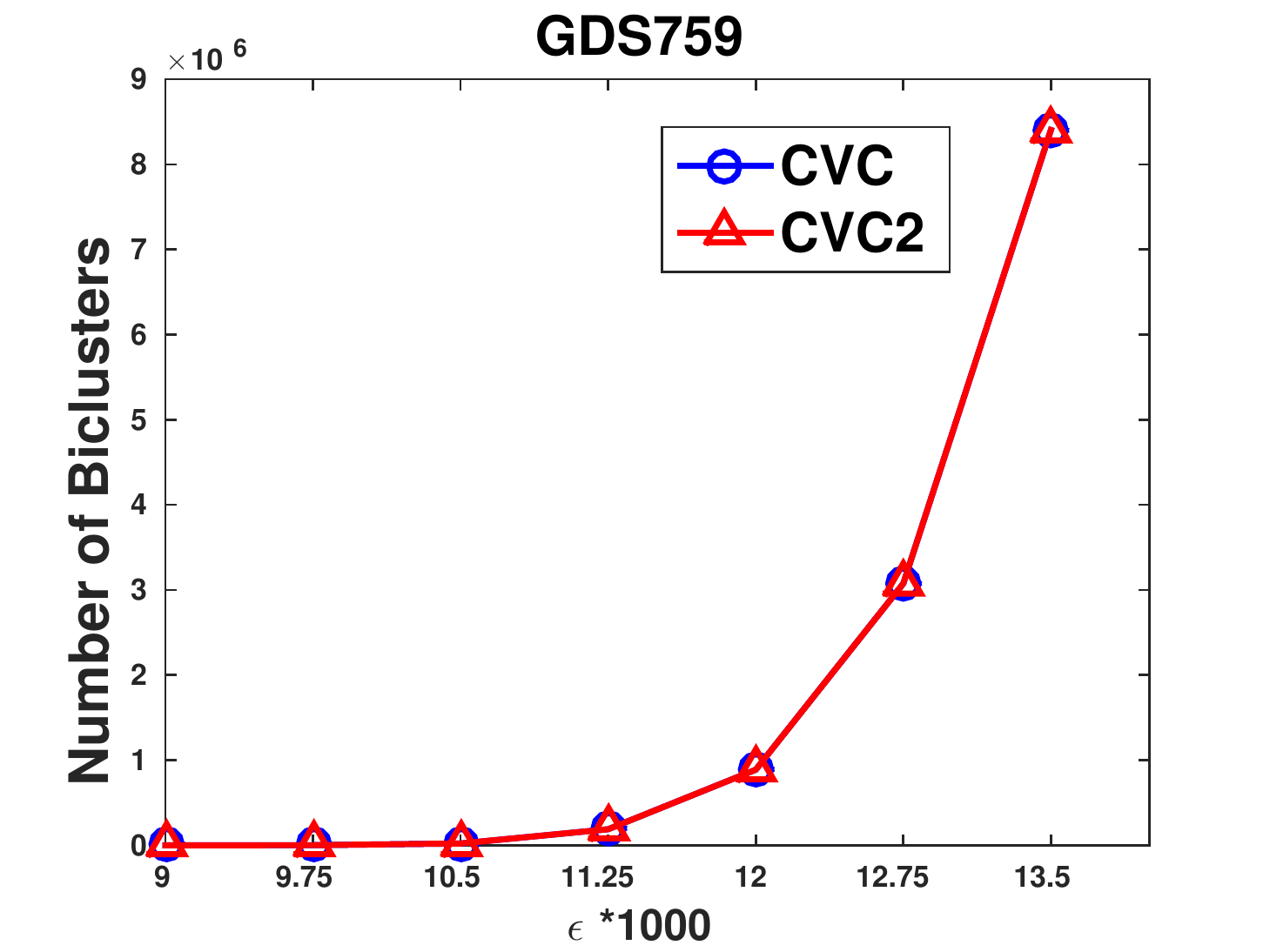}
}
\subfigure[]{
  \includegraphics[trim=0.4cm 0.1cm 1.4cm 0.15cm, clip, scale=0.35]{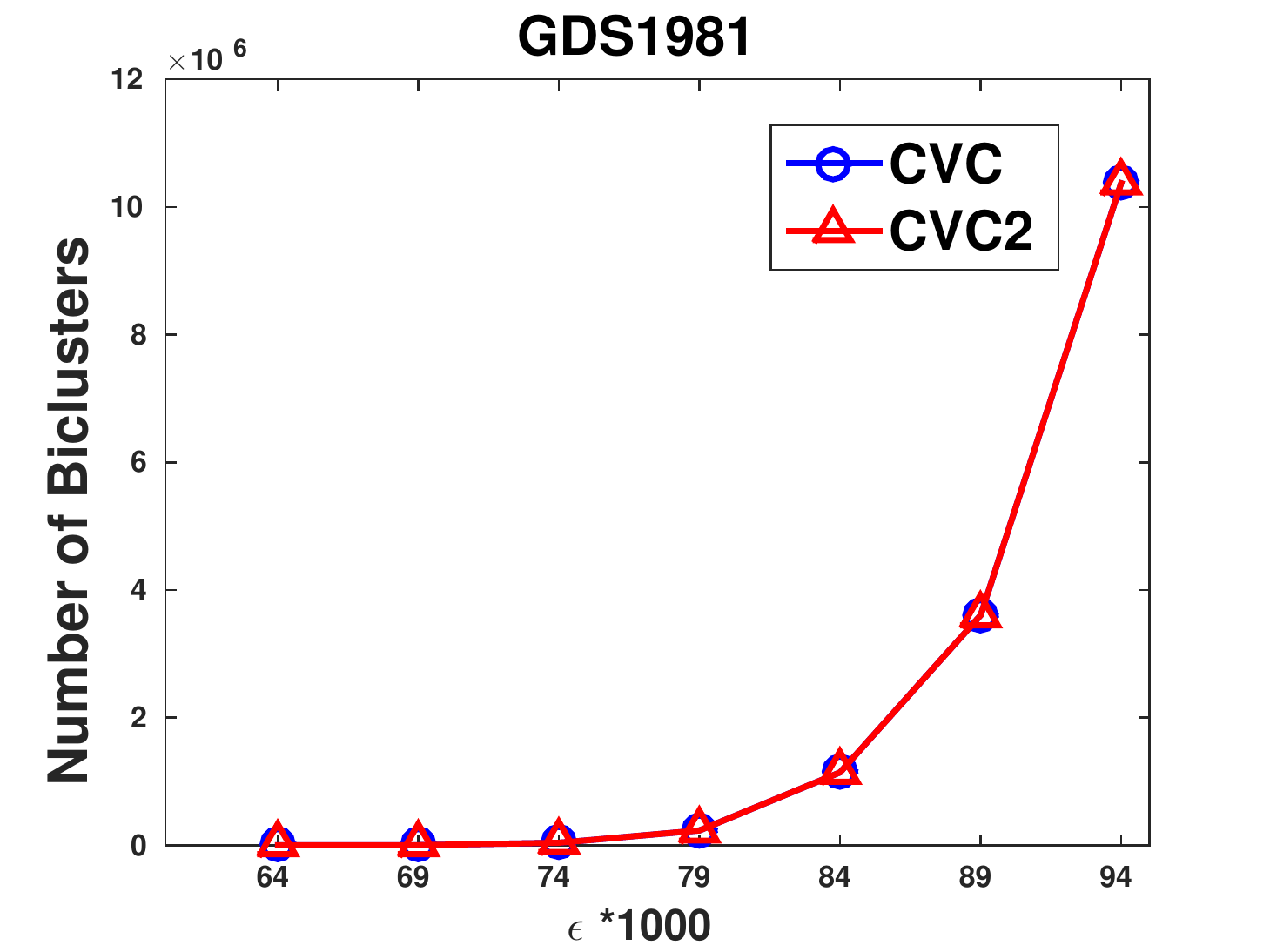}
}
\subfigure[]{
  \includegraphics[trim=0.4cm 0.1cm 1.4cm 0.15cm, clip, scale=0.35]{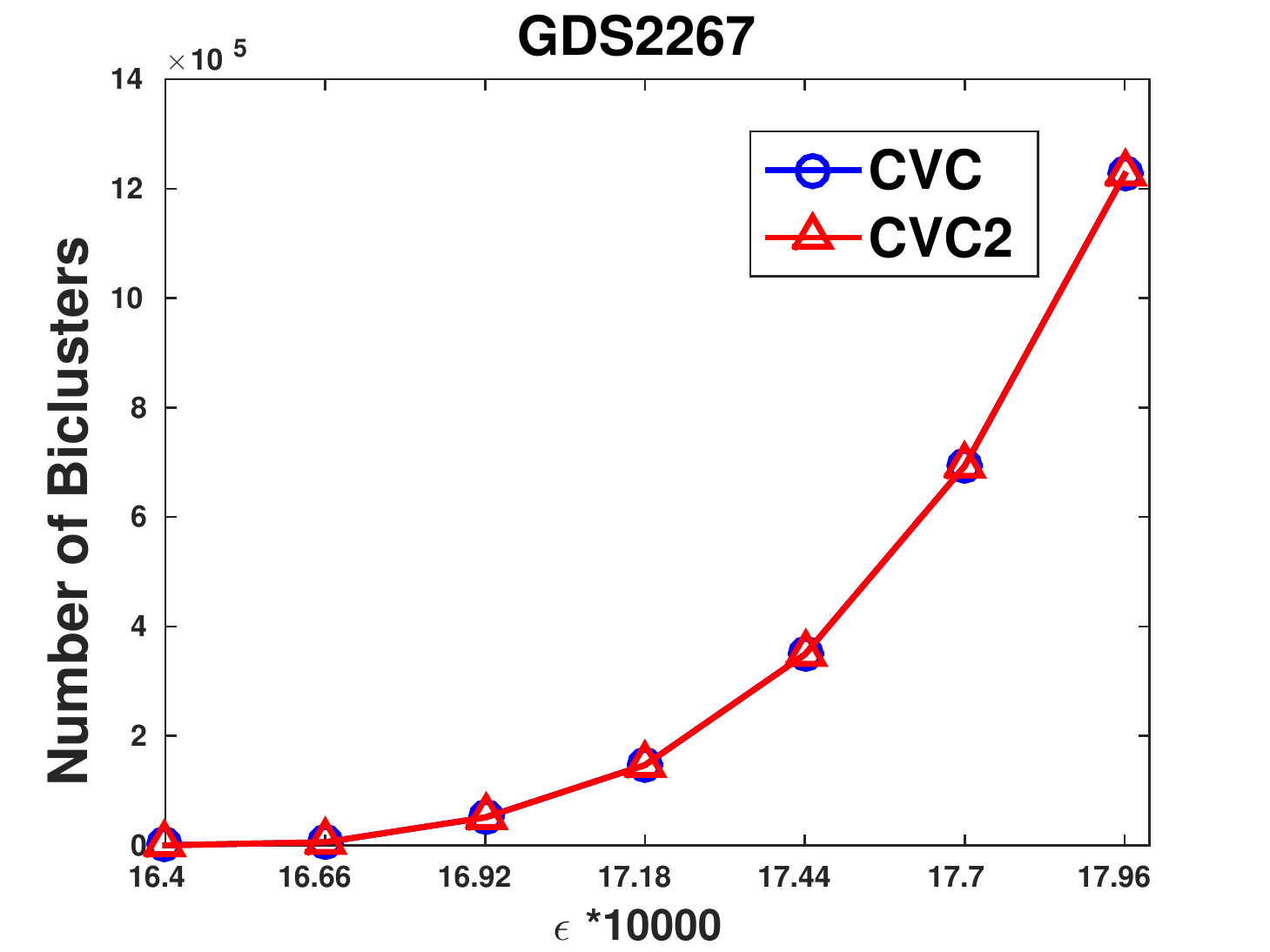}
}
\subfigure[]{
  \includegraphics[trim=0.4cm 0.1cm 1.4cm 0.15cm, clip, scale=0.35]{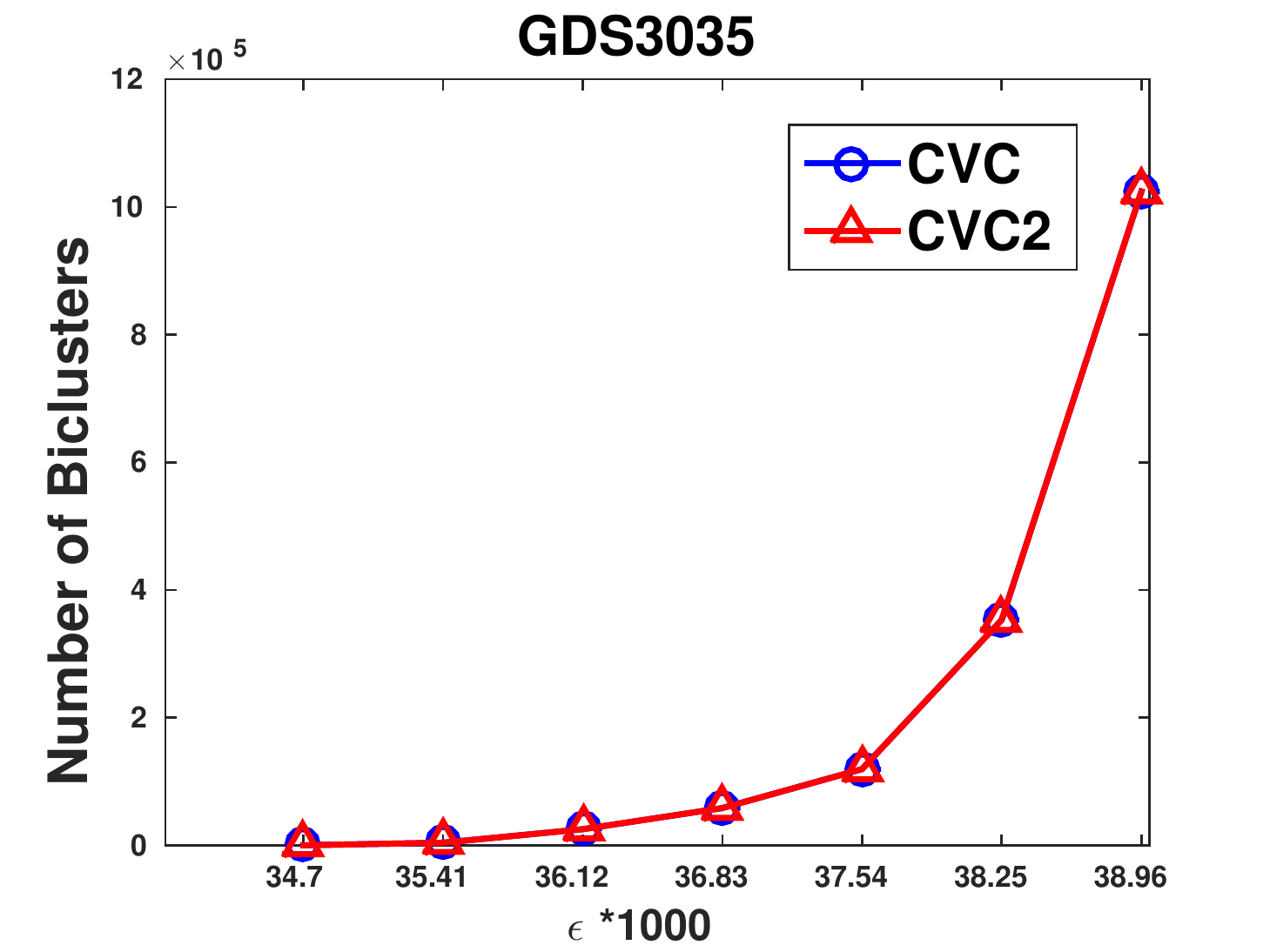}
}
\caption{Number of biclusters when varying the user-defined parameter $\epsilon$ for the datasets (a) GDS750, (b) GDS759, (c) GDS1981, (d) GDS2267, and (e) GDS3035.}
\label{fig:expRealDataNBIC}
\end{figure*}

\begin{figure*}
\centering
\subfigure[]{
  \includegraphics[trim=0.15cm 0.1cm 1.4cm 0.15cm, clip, scale=0.35]{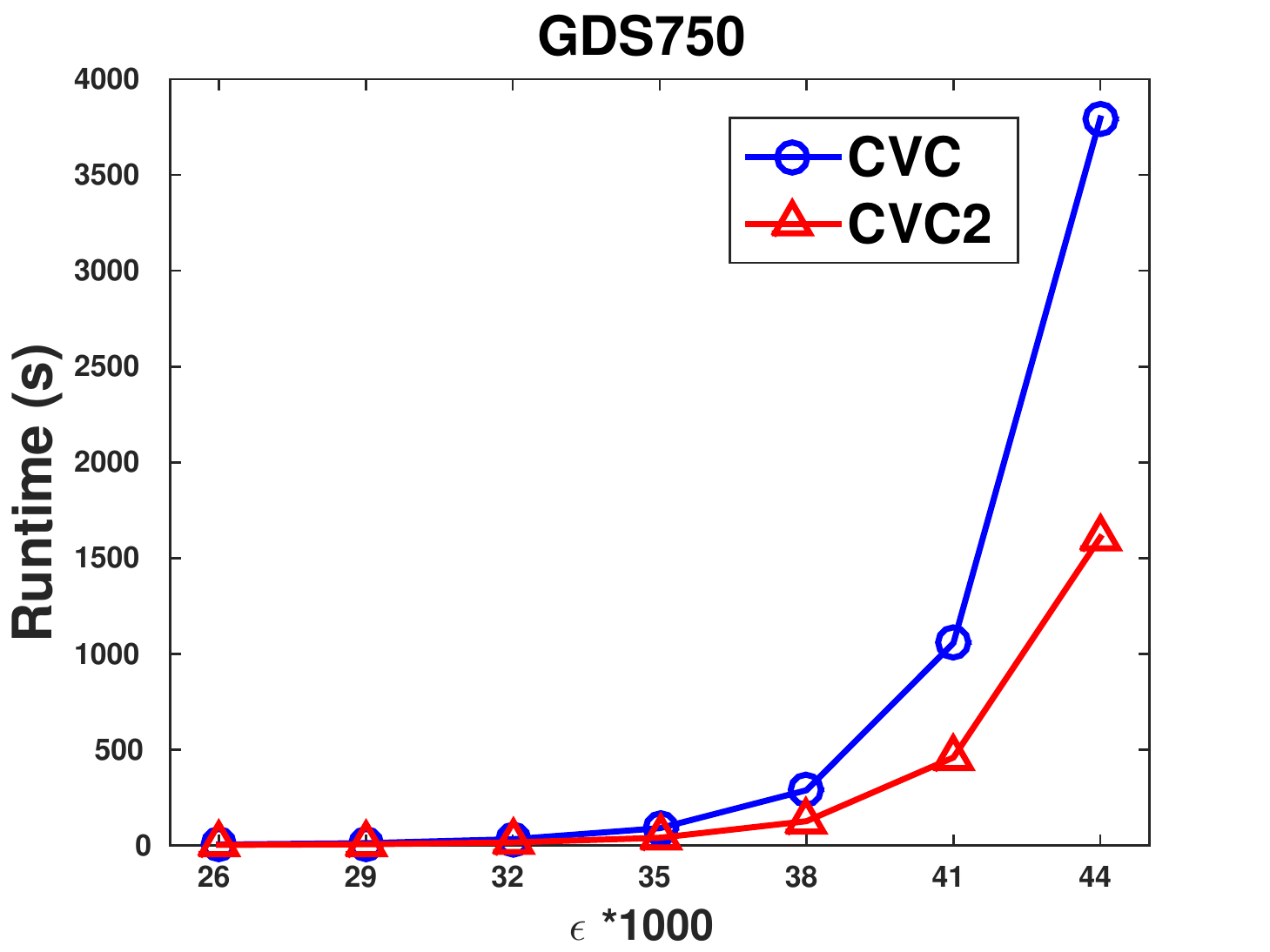}
}
\subfigure[]{
  \includegraphics[trim=0.15cm 0.1cm 1.4cm 0.15cm, clip, scale=0.35]{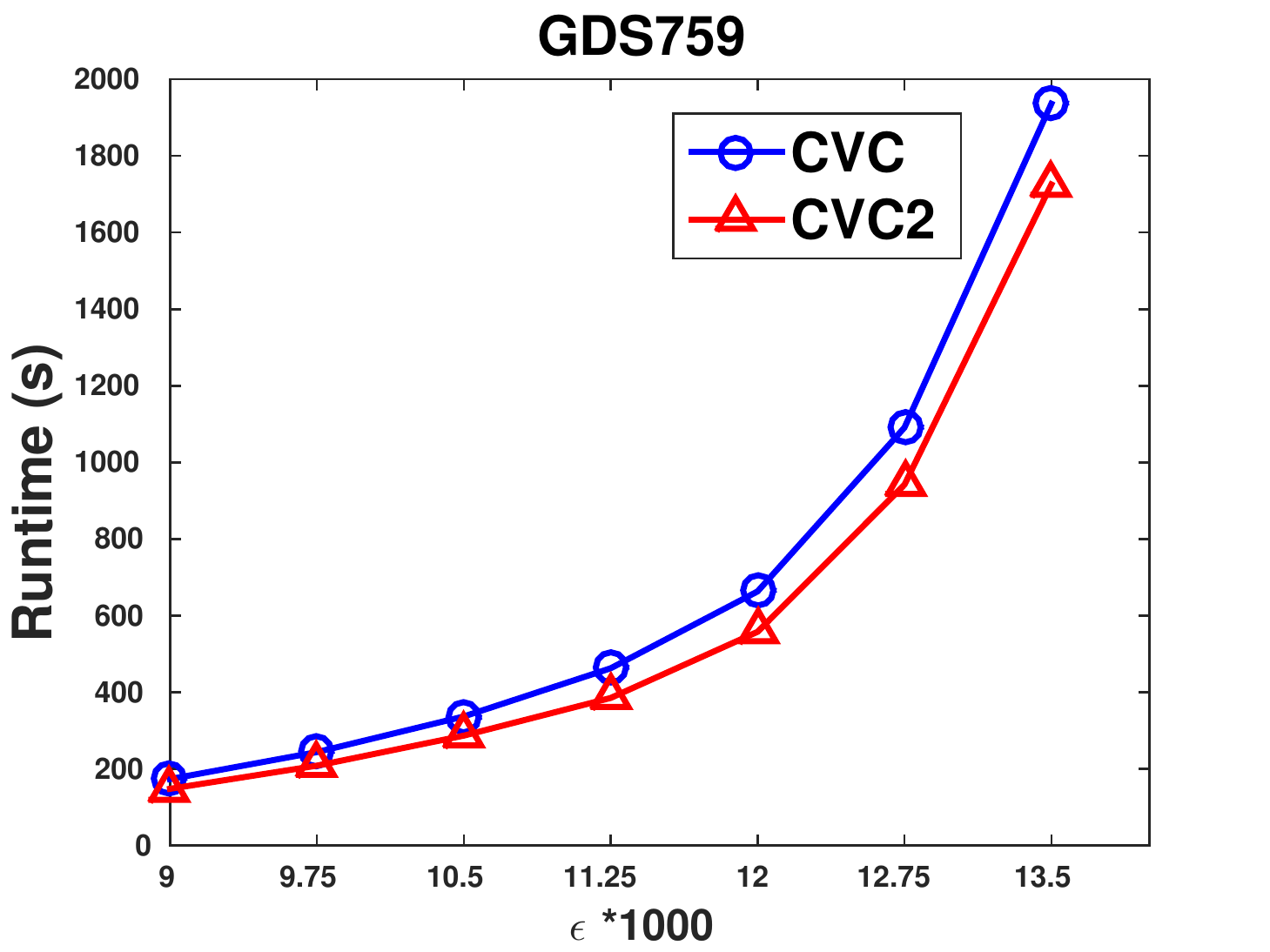}
}
\subfigure[]{
  \includegraphics[trim=0.15cm 0.1cm 1.4cm 0.15cm, clip, scale=0.35]{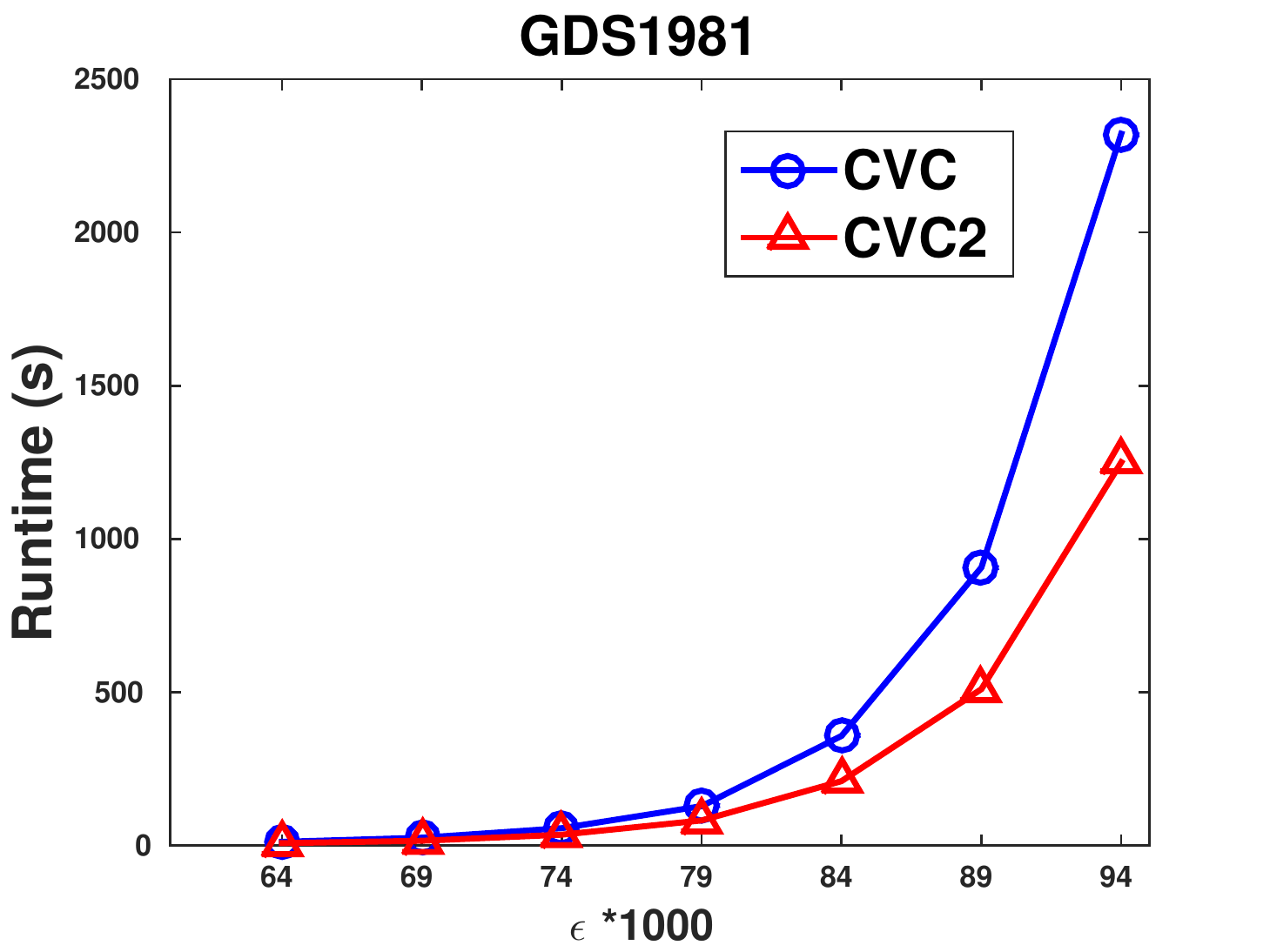}
}
\subfigure[]{
  \includegraphics[trim=0.15cm 0.1cm 1.4cm 0.15cm, clip, scale=0.35]{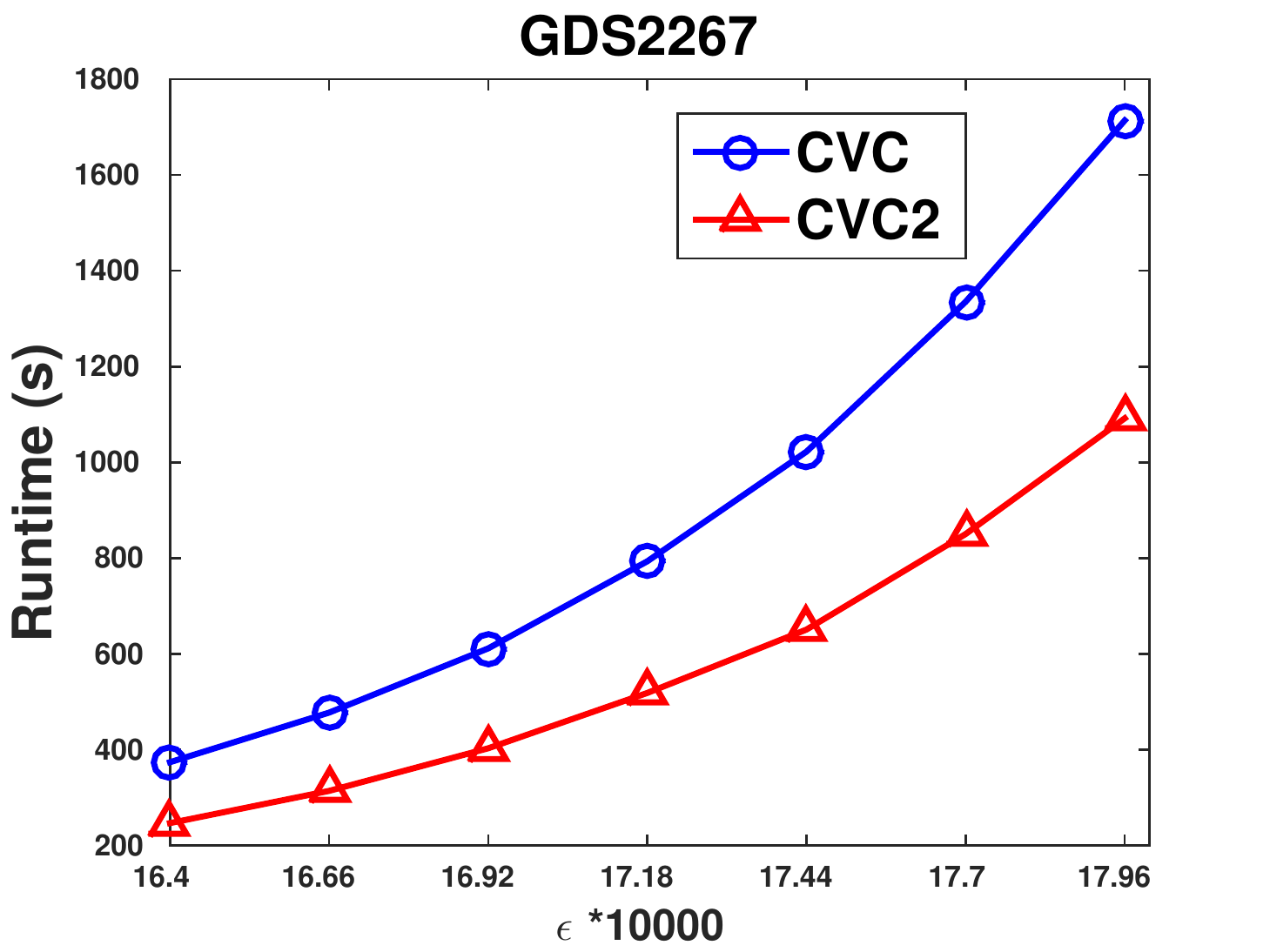}
}
\subfigure[]{
  \includegraphics[trim=0.15cm 0.1cm 1.4cm 0.15cm, clip, scale=0.35]{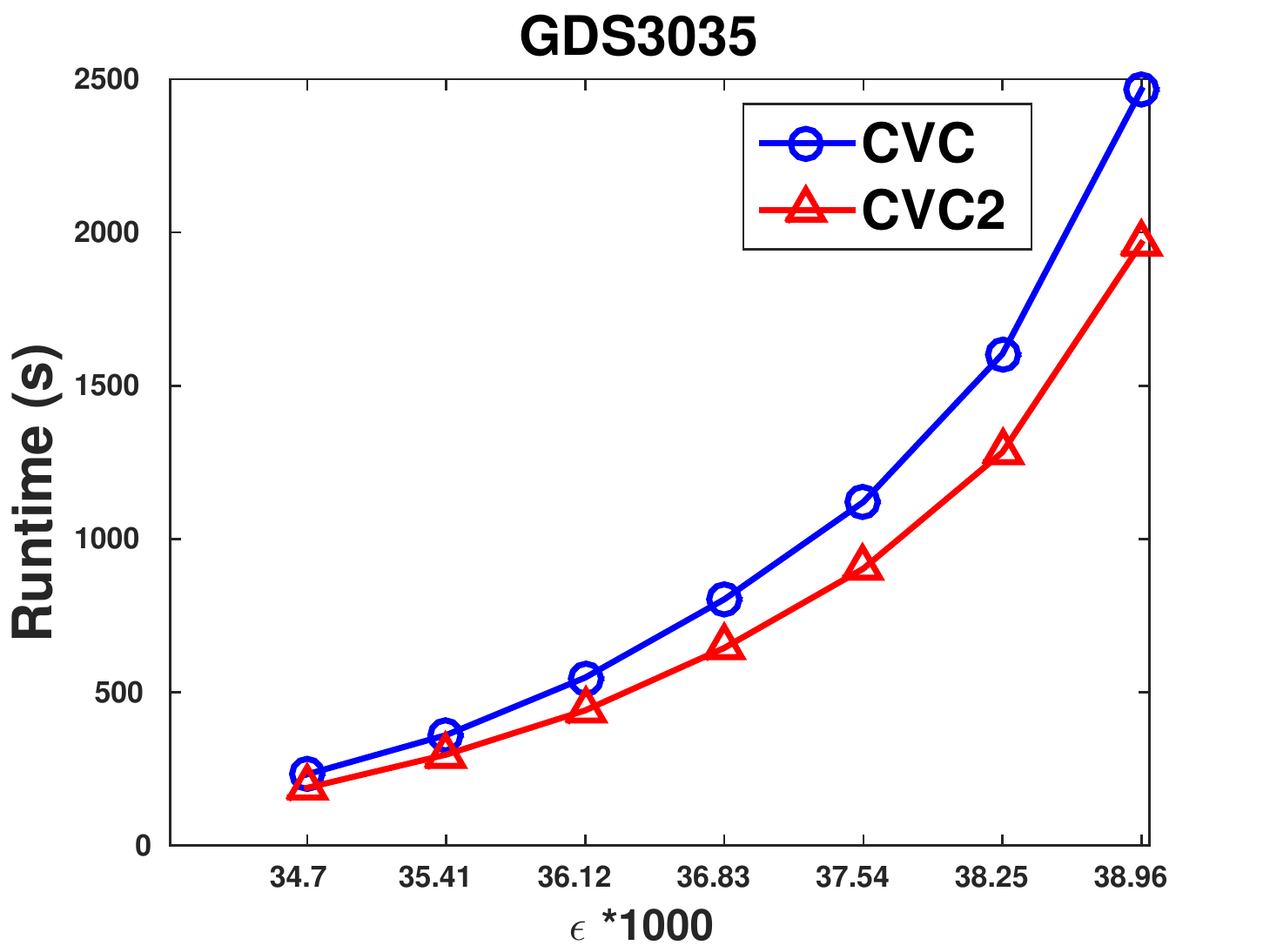}
}
\caption{Runtime of RIn-Close\_CVC and RIn-Close\_CVC2 when varying the user-defined parameter $\epsilon$ (which controls the maximum perturbation of the biclusters) for the datasets (a) GDS750, (b) GDS759, (c) GDS1981, (d) GDS2267, and (e) GDS3035.}
\label{fig:expRealDataRT}
\end{figure*}

\begin{figure*}
\centering
\subfigure[]{
  \includegraphics[trim=0.4cm 0.1cm 1.4cm 0.2cm, clip, scale=0.35]{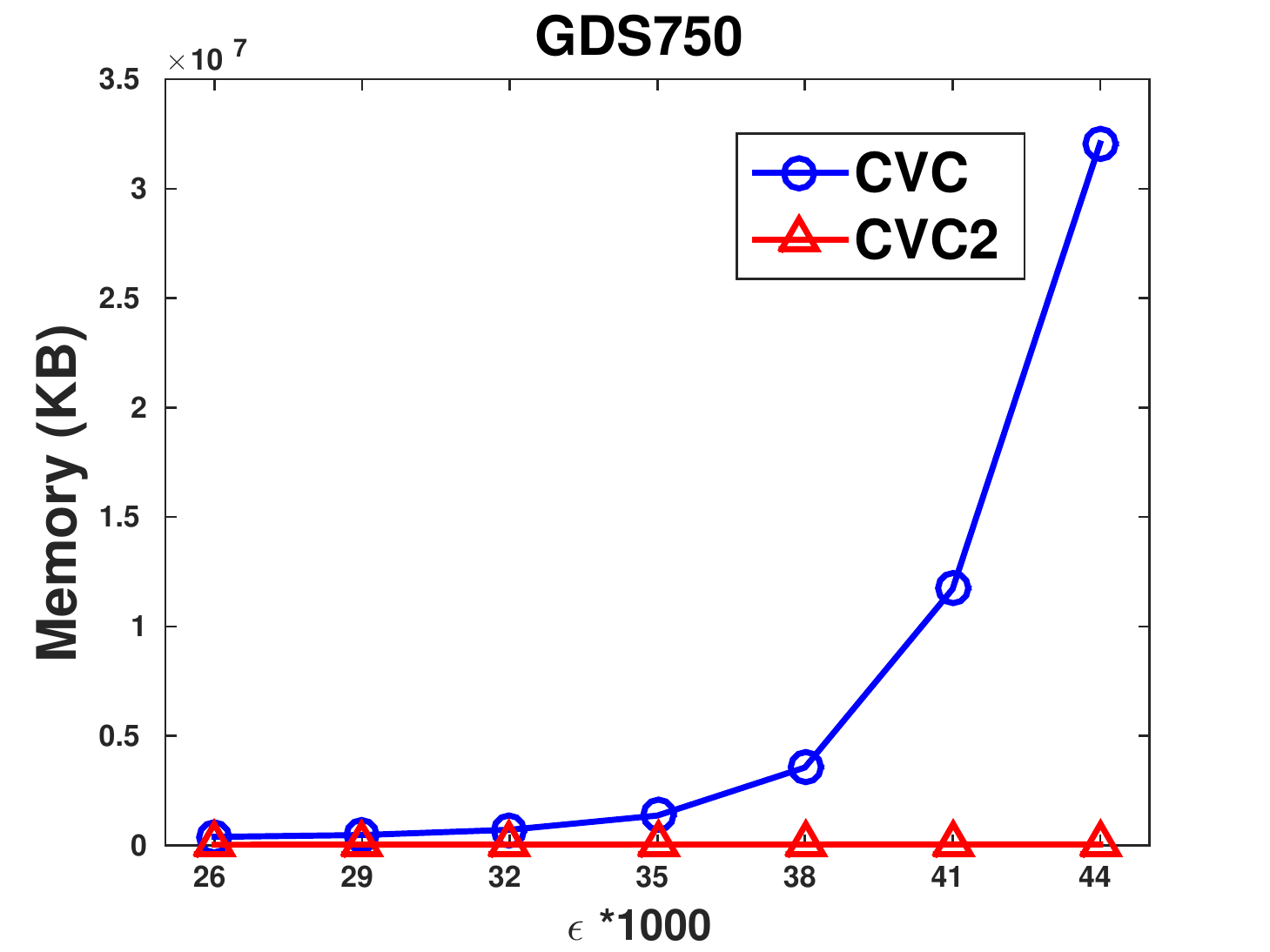}
}
\subfigure[]{
  \includegraphics[trim=0.4cm 0.1cm 1.4cm 0.2cm, clip, scale=0.35]{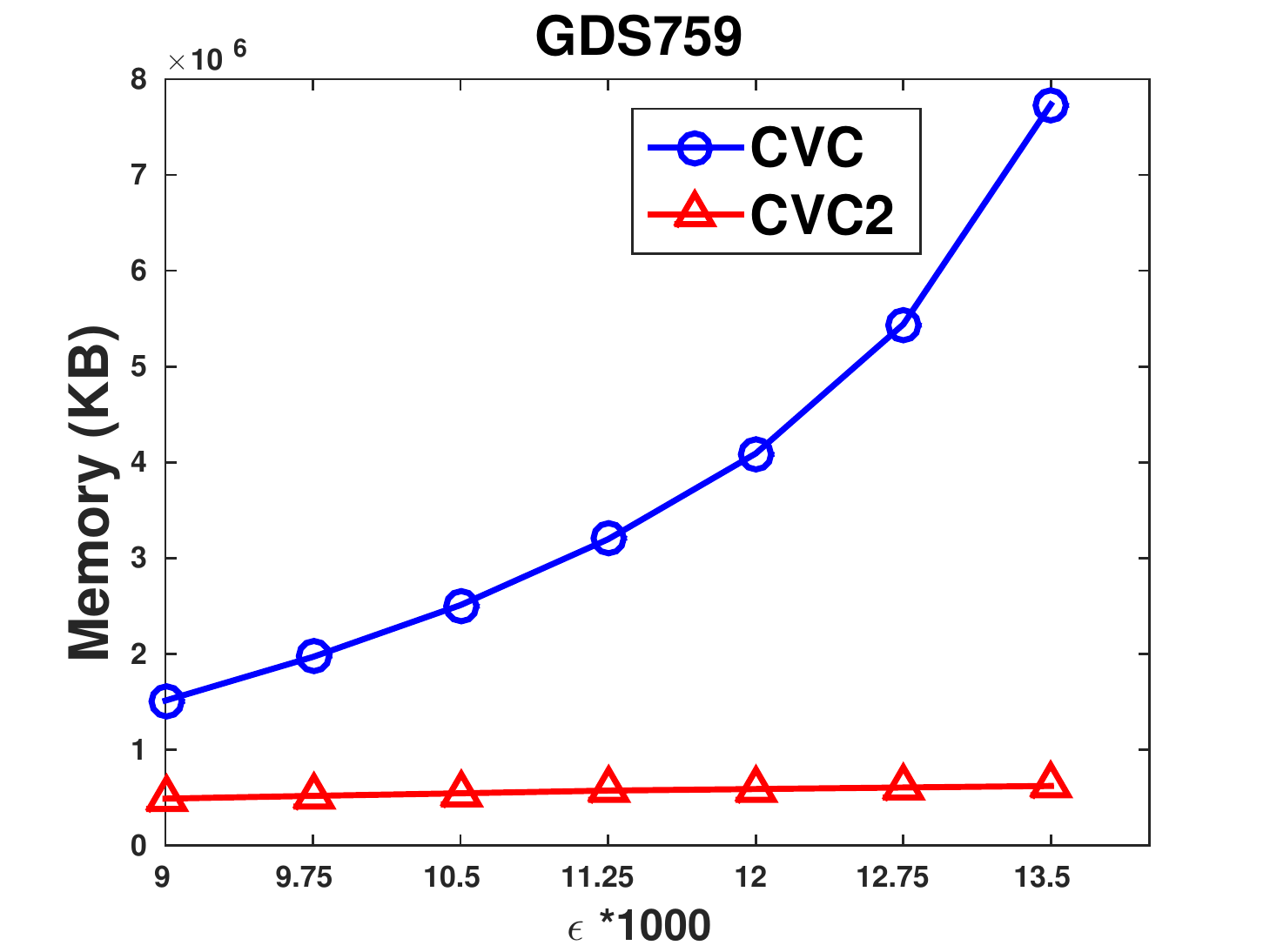}
}
\subfigure[]{
  \includegraphics[trim=0.4cm 0.1cm 1.4cm 0.2cm, clip, scale=0.35]{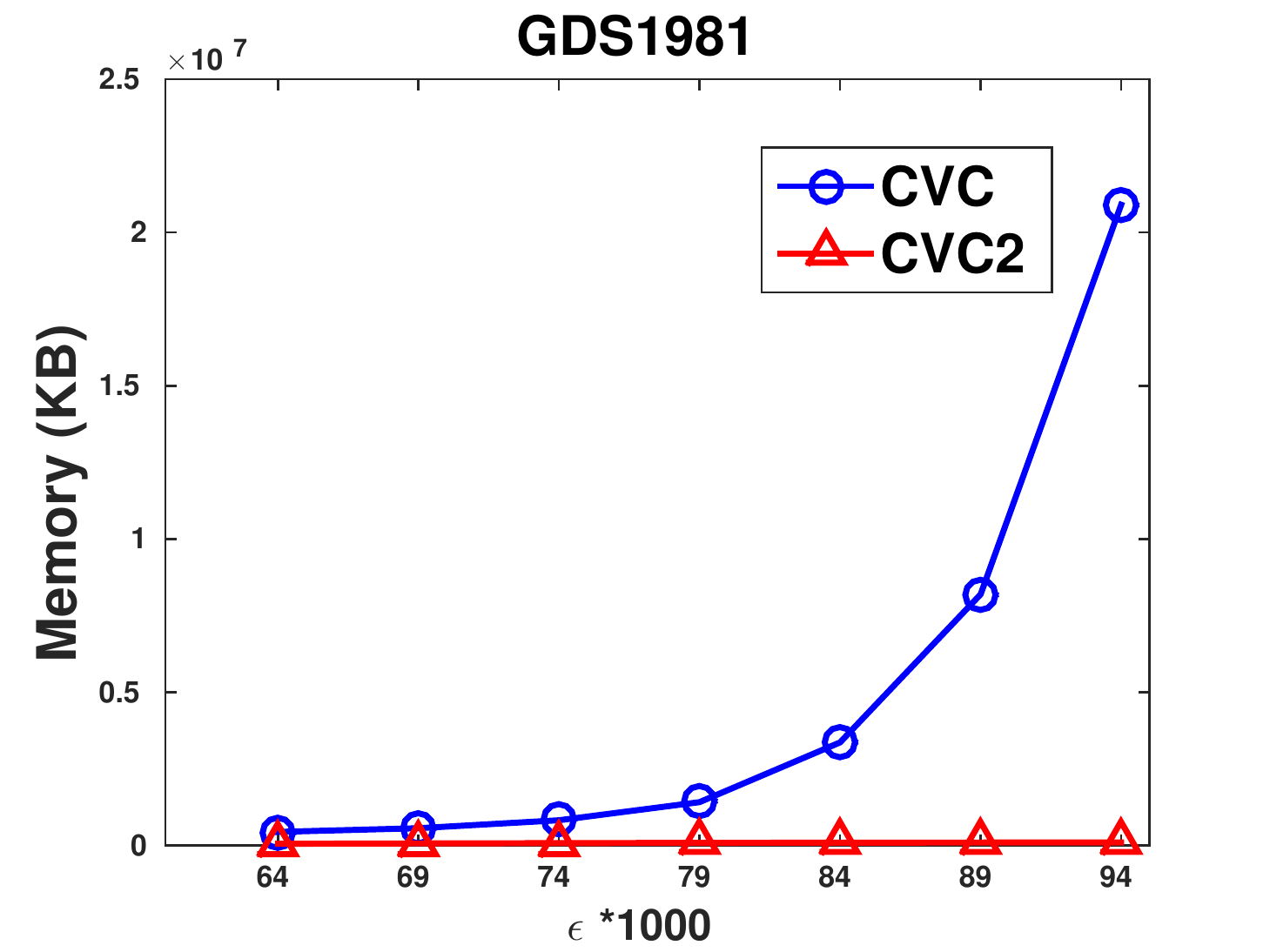}
}
\subfigure[]{
  \includegraphics[trim=0.4cm 0.1cm 1.4cm 0.2cm, clip, scale=0.35]{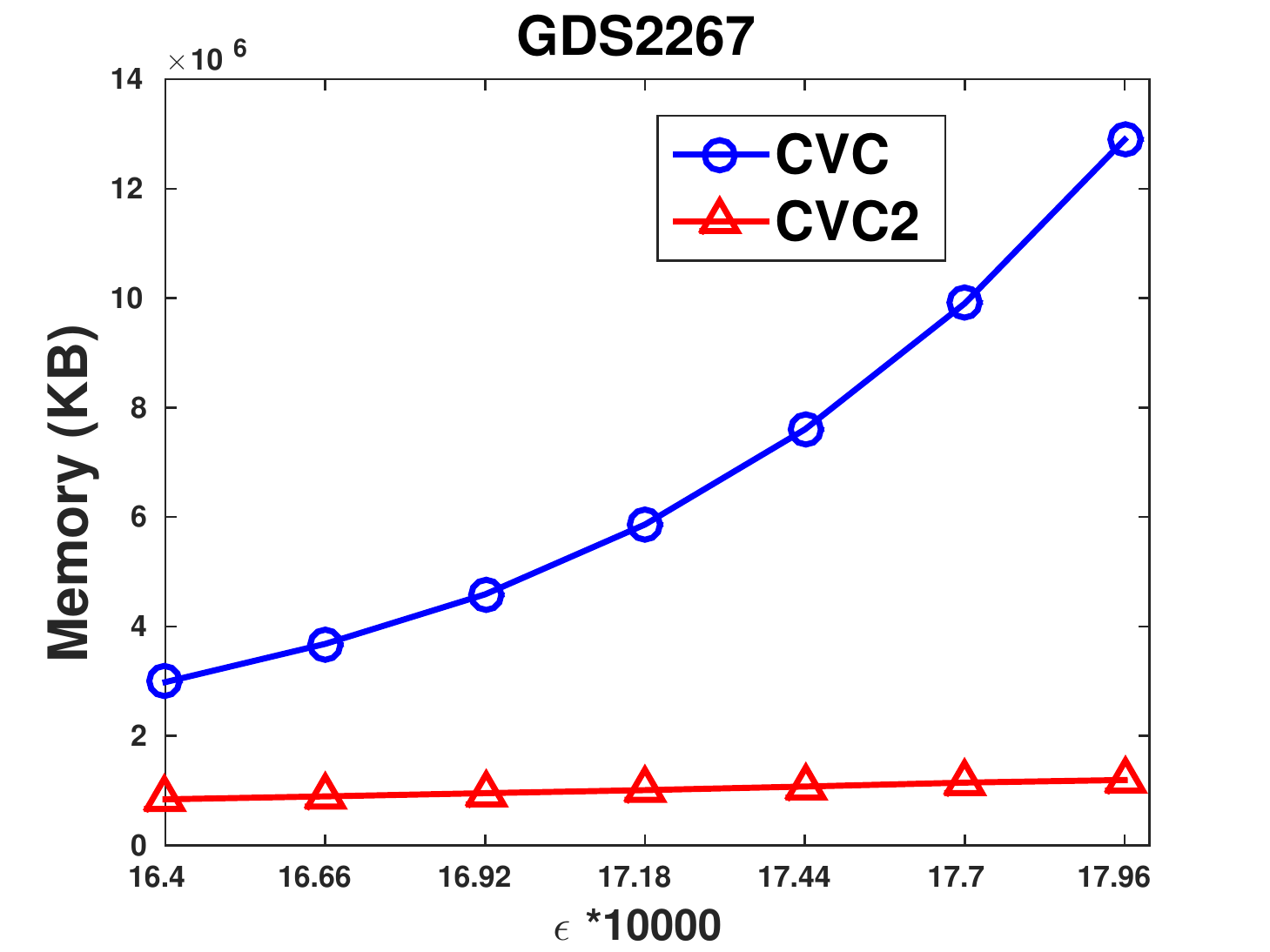}
}
\subfigure[]{
  \includegraphics[trim=0.4cm 0.1cm 1.4cm 0.2cm, clip, scale=0.35]{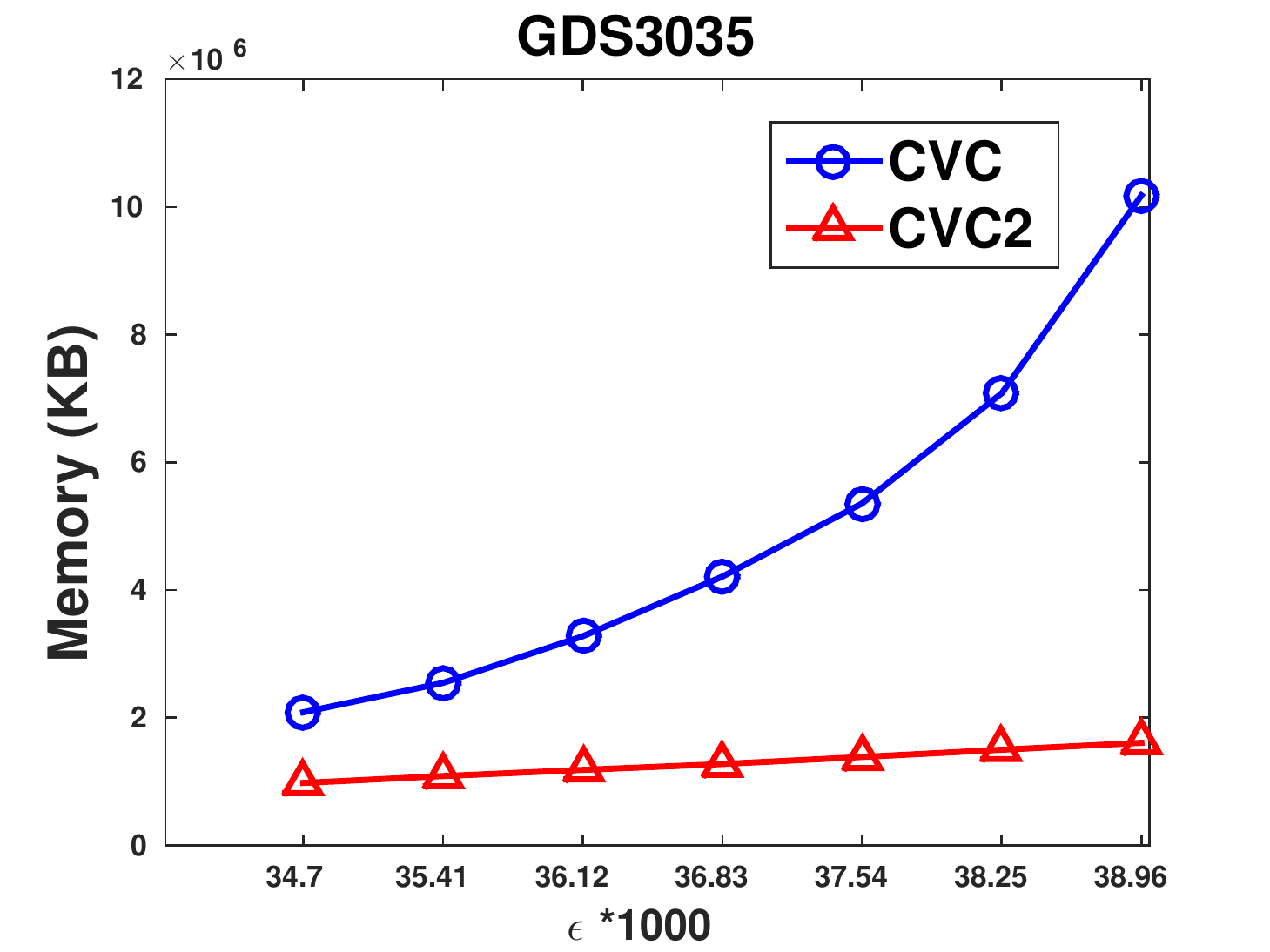}
}
\caption{Memory usage of RIn-Close\_CVC and RIn-Close\_CVC2 when varying the user-defined parameter $\epsilon$ (which controls the maximum perturbation of the biclusters) for the datasets (a) GDS750, (b) GDS759, (c) GDS1981, (d) GDS2267, and (e) GDS3035.}
\label{fig:expRealDataMEM}
\end{figure*}

\section{Concluding Remarks}
\label{sec:conclusion}

In this paper, we proposed a new version of RIn-Close\_CVC, named RIn-Close\_CVC2, which brings a large reduction in memory usage, and also, in average, significant runtime gain, even having a higher worst-case time-complexity. The new algorithm also keeps the four key properties of its predecessor: efficiency, completeness, correctness and non-redundancy. We proved it here, after demonstrating that RIn-Close\_CVCP and RIn-Close\_CVC also exhibit these four properties.

The results of our experiments with synthetic datasets showed that the memory usage of RIn-Close\_CVC2 was equivalent to the memory usage of RIn-Close\_CVCP, known to be very attractive. Also, the results of our experiments with real-world datasets showed that RIn-Close\_CVC2 presented a linear growth in the memory usage, even though the number of biclusters exhibited an exponential growth with the value of admissible residue $\epsilon$. They are great achievements, opening the possibility of enumerating perturbed maximal CVC biclusters in previously infeasible scenarios. Moreover, the new version also achieved better results in terms of runtime in our experiments. Thus, the new version dominates its predecessor in all relevant scalability aspects.

In future works, we are planning to test another initiative that may promote gain in the computational performance of RIn-Close algorithms: to incorporate the distinctive aspects of the recently proposed In-Close4 \cite{Andrews2017} and In-Close5 \cite{Andrews2018} algorithms, both restricted to deal with binary datasets, into our RIn-Close family of algorithms, which are devoted to mine biclusters in numerical datasets and were based on In-Close2 \cite{Andrews2009}.

\section*{Acknowledgments}

R. Veroneze and F. J. Von Zuben would like to thank FAPESP (Process Number: 2017/21174-8), CAPES and CNPq (Process Number 309115/2014-0) for the financial support.

\section*{References}


{\footnotesize
\bibliography{tese}}

\vfill
{\footnotesize
\textbf{Rosana Veroneze} is a postdoctoral researcher at the Department of Computer Engineering and Industrial Automation, School of Electrical and Computer Engineering, University of Campinas (Unicamp). Her research interests include computational intelligence, data mining and machine learning areas.\\

\textbf{Fernando J. Von Zuben} is a Full Professor at the Department of Computer Engineering and Industrial Automation, School of Electrical and Computer Engineering, University of Campinas (Unicamp). The main topics of his research are computational intelligence, bioinspired computing, multivariate data analysis, and machine learning.

}
\end{document}